\documentclass[twoside,11pt]{article}
\usepackage{amsmath}
\usepackage{bm}
\usepackage{jmlr2e}
\usepackage{amsmath,mathrsfs}
\usepackage{multirow}
\usepackage{paralist,algorithmic,algorithm}

\newcommand{\qed}{\hfill\BlackBox}
\DeclareMathOperator{\sgn}{sgn}

\def \x {\mathbf{x}}
\def \w {\mathbf{w}}

\def \r {\mathbf{r}}
\def \c {\mathbf{c}}
\def \R {\mathbb{R}}

\def \SK {\widehat{K}}
\def \Sh {\widehat{S}}

\def \gh {\widehat{g}}
\def \L  {\mathcal{L}}
\def \wh {\widehat{\w}}
\def \Lh {\widehat{\mathcal{L}}}
\def \ch {\widehat{\c}}
\def \a {\mathbf{a}}
\def \At {\widetilde{A}}

\jmlrheading{}{}{}{}{}{Gao, Jin, Zhu and Zhou}

\ShortHeadings{One-Pass AUC Optimization}{Gao, Jin, Zhu and Zhou} \firstpageno{1}

\begin{document}

\title{One-Pass AUC Optimization}

\author{\name Wei Gao \email gaow@lamda.nju.edu.cn\\
\addr National Key Laboratory for Novel Software Technology\\
Nanjing University, Nanjing 210023, China \\
\AND Rong Jin \email rongjin@cse.msu.edu\\
\addr Department of Computer Science and Engineering,\\
Michigan State University, East Lansing, MI, 48824, USA \\
\AND Shenghuo Zhu \email zsh@nec-labs.com\\
\addr NEC Laboratories America, CA, 95014, USA\\
\AND Zhi-Hua Zhou \email zhouzh@lamda.nju.edu.cn\\
\addr National Key Laboratory for Novel Software Technology\\
Nanjing University, Nanjing 210023, China \\
       }

\editor{}

\maketitle

\begin{abstract}
AUC is an important performance measure and many algorithms have been devoted to AUC optimization, mostly by minimizing a surrogate convex loss on a training data set. In this work, we focus on one-pass AUC optimization that requires going through the training data only once without storing the entire training dataset, where conventional online learning algorithms cannot be applied directly because AUC is measured by a sum of losses defined over pairs of instances from different classes. We develop a regression-based algorithm which only needs to maintain the first and second-order statistics of training data in memory, resulting a storage requirement independent from the size of training data. To efficiently handle high-dimensional data, we develop a randomized algorithm that approximates the covariance matrices by low-rank matrices. We verify, both theoretically and empirically, the effectiveness of the proposed algorithm.
\end{abstract}

\begin{keywords}
AUC optimization, learning to rank, large-scale learning, random projection, square loss
\end{keywords}

\section{Introduction}

AUC (Area Under ROC curve) \citep{Metz1978,Hanley:McNeil1983} is an important performance measure that has been widely used in many tasks \citep{Provost:Fawcett:Kohavi1998,Cortes:Mohri2004,Liu:Wu:Zhou2009,Flach:Orallo:Ramirez2011}. Many algorithms have been developed to optimize AUC based on surrogate losses~\citep{Herschtal:Raskutti2004,Joachims2006,Rudin:Schapire2009,Kotlowski:Dembczynski:Hullermeier2011,Zhao:Hoi:Jin:Yang2011}.

In this work, we focus on AUC optimization that requires only one pass of training examples. This is particularly important for applications involving big data or streaming data in which a large volume of data come in a short time period, making it infeasible to store the entire data set in memory before an optimization procedure is applied. Although many online learning algorithms have been developed to find the optimal solution of some performance measures by only scanning the training data once~\citep{Bianchi:Lugosi2006}, few effort addresses one-pass AUC optimization.

Unlike the classical classification and regression problems where the loss function can be calculated on a single training example, AUC is measured by the losses defined over pairs of instances from different classes, making it challenging to develop algorithms for one-pass optimization. An online AUC optimization algorithm was proposed very recently by~\citet{Zhao:Hoi:Jin:Yang2011}. It is based on the idea of reservoir sampling, and achieves a solid regret bound by only storing $\sqrt{T}$ instances, where $T$ is the number of training examples. Ideally, for one-pass approaches, it is crucial that the storage required by the learning process should be independent from the amount of training data, because it is often quite difficult to expect how many data will be received in those applications.

In this work, we propose a regression-based algorithm for one-pass AUC optimization in which a square loss is used to measure the ranking error between two instances from different classes. The main advantage of using the square loss lies in the fact that it only needs to store the first and second-order statistics for the received training examples. Consequently, the storage requirement is reduced to $O(d^2)$, where $d$ is the dimension of data, independent from the number of training examples. To deal with high-dimensional data, we develop a randomized algorithm that approximates the covariance matrix of $d\times d$ by a low-rank matrix. We show, both theoretically and empirically, the effectiveness of our proposal algorithm by comparing to state-of-the-art algorithms for AUC optimization.

Section~\ref{sec:Pre} introduces some preliminaries. Sections~\ref{sec:alg} proposes the OPAUC (One Pass AUC) framework, and Section~\ref{sec:theory} provides theoretical analysis and Section~\ref{sec:pf} presents detailed proofs. Section~\ref{sec:exp} summaries our experimental results. Section~\ref{sec:con} concludes with future work.

\section{Preliminaries}\label{sec:Pre}

We denote by $\mathcal{X}\in\mathbb{R}^d$ an instance space and $\mathcal{Y}=\{+1,-1\}$ the label set, and let $\mathcal{D}$ denote an unknown (underlying) distribution over $\mathcal{X}\times\mathcal{Y}$. A training sample of $n_+$ positive instances and $n_-$ negative ones
\[
\mathcal{S} =\{(\x^+_1,+1),(\x^+_2,+1),\ldots,(\x^+_{n_+},+1), (\x^-_1,-1), (\x^-_2,-1), \ldots,(\x^-_{n_-},-1)\}
\]
is drawn identically and independently according to distribution $\mathcal{D}$, where we do not fix $n_+$ and $n_-$ before the training sample is chosen. Let $f\colon\mathcal{X}\to\mathbb{R}$ be a real-valued function. Then, the AUC of function $f$ on the sample $\mathcal{S}$ is defined as
\[
\sum_{i=1}^{n_+}\sum_{j=1}^{n_-} \frac{\mathbb{I}[f(\x_i^+)>f(\x^-_j)]+\frac{1}{2}\mathbb{I}[f(\x_i^+)=f(\x^-_j)]}{n_+n_-}
\]
where $\mathbb{I}[\cdot]$ is the indicator function which returns $1$ if the argument is true and $0$ otherwise.

Direct optimization of AUC often leads to an NP-hard problem as it can be cast into a combinatorial optimization problem. In practice, it is approximated by a convex optimization problem that minimizes the following objective function
\begin{equation}\label{eq:surrogateloss}
\mathcal{L}(\w) = \frac{\lambda}{2}|\w|^2 +  \sum_{i=1}^{n_+}\sum_{j=1}^{n_-} \frac{\ell\left(\w^{\top}(\x^+_i - \x^-_j)\right)}{2n_+n_-}
\end{equation}
where $\ell$ is a convex loss function and $\lambda$ is the regularization parameter that controls the model complexity. Notice that each loss term $\ell(\w^{\top}(\x_i^+ - \x_j^-))$ involves two instances from different classes; therefore, it is difficult to extend online learning algorithms for one-pass AUC optimization without storing all the training instances. \citet{Zhao:Hoi:Jin:Yang2011} addressed this challenge by exploiting the reservoir sampling technique.

\section{The OPAUC Approach}\label{sec:alg}

To address the challenge of one-pass AUC optimization, we propose to use the square loss in Eq.~\eqref{eq:surrogateloss}, that is,
\begin{equation}\label{eq:leastqaure}
\mathcal{L}(\w) = \frac{\lambda}{2}|\w|^2 + \sum_{i=1}^{n_+}\sum_{j=1}^{n_-} \frac{(1-\w^{\top}(\x^+_i - \x^-_j))^2}{2n_+n_-}.
\end{equation}

The main advantage of using the square loss lies in the fact that it is sufficient to store the first and second-order statistics of training examples for optimization, leading to a memory requirement of $O(d^2)$, which is independent from the number of training examples. Another advantage is that the square loss is consistent with AUC, as will be shown by Theorem~\ref{thm:con} (Section~\ref{sec:theory}). In contrast, loss functions such as hinge loss are proven to be inconsistent with AUC~\citep{Gao:Zhou2012}.

As aforementioned, the classical online setting cannot be applied to one-pass AUC optimization because, even if the optimization problem of Eq.~\eqref{eq:leastqaure} has a closed form, it requires going through the training examples multiple times. To address this challenge, we modify the overall loss $\mathcal{L}(\w)$ in Eq.~\eqref{eq:leastqaure} (with a little variation) as a sum of losses for individual training instance $\sum_{t=1}^T \mathcal{L}_t(\w)$, where
\[
\mathcal{L}_t(\w)= \frac{\lambda}{2}|\w|^2 + \frac{\sum_{i=1}^{t-1}\mathbb{I}[y_i\neq y_t] (1-y_t(\x_t-\x_i)^\top\w)^2}{2|\{i\in[t-1]:y_iy_t=-1\}|}
\]
for sequence $\mathcal{S}_t= \{(\x_1,y_1), \ldots, (\x_t,y_t)\}$. It is noteworthy that $\mathcal{L}_t(\w)$ is an unbiased estimation to $\mathcal{L}(\w)$ for i.i.d. sequence $\mathcal{S}_t$.  For notational simplicity, we denote by $X_t^+$ and $X_t^-$ the sets of positive and negative instances in the sequence $\mathcal{S}_t$, respectively, and we further denote by $T_t^+$ and $T_t^-$ their respective cardinalities. Also, we set $\mathcal{L}_t(\w)=0$ for $T_t^+T_t^-=0$.

If $y_t=1$, we calculate the gradient as
\begin{equation}\label{eqn:grad}
\nabla \mathcal{L}_t(\w)=\lambda\w + \x_t\x_t^\top \w-\x_t
+\sum_{i\colon y_i=-1}\frac{\x_i + (\x_i\x_i^\top-\x_i\x_t^\top-\x_t\x_i^\top)\w}{T_t^-}.
\end{equation}
It is easy to observe that
\[
c_t^-=\sum_{i\colon y_i=-1}\frac{\x_i}{T_t^-} \quad\text{ and } \quad S_t^-= \sum_{i\colon y_i=-1} \frac{\x_i\x_i^\top - \c_t^-[\c_t^-]^{\top}}{T_t^-}
\]
correspond to the mean and covariance matrix of negative class, respectively; thus, Eq.~\eqref{eqn:grad} can be further simplified as
\begin{equation}\label{eqn:grad-minus-ls}
\nabla \mathcal{L}_t(\w)=\lambda \w-\x_t +  \c_{t}^{-} + (\x_t - \c_t^-)(\x_t - \c_t^-)^\top\w + S_{t}^{-}\w.
\end{equation}
In a similar manner, we calculate the following gradient for $y_t=-1$:
\begin{equation}\label{eqn:grad-plus-ls}
\nabla \mathcal{L}_t(\w)=\lambda \w + \x_t-\c_{t}^{+}+(\x_t - \c_t^+)(\x_t - \c_t^+)^\top\w + S_{t}^{+}\w
\end{equation}
where
\[
\c_t^+=\sum_{i\colon y_i=1}\frac{\x_i}{T_t^+} \quad\text{ and }\quad S_t^+= \sum_{i\colon y_i=1} \frac{\x_i\x_i^\top - \c_t^{+} [\c_t^+]^{\top}}{T_t^+}
\]
are the covariance matrix and mean of positive class, respectively.

The storage cost for keeping the class means ($\c_{t}^+$ and $\c_{t}^-$) and covariance matrices  ($S_{t-1}^+$ and $S_{t-1}^-$) is $O(d^2)$. Once we get the gradient $\nabla \mathcal{L}_t(\w)$, by theory of stochastic gradient descent, the solution can be updated by
\[
\w_{t+1}= \w_t - \eta_t \nabla \mathcal{L}_t(\w_t)
\]
where $\eta_t$ is the stepsize for the $t$-th iteration.

\begin{algorithm}[t]
\caption{The OPAUC Algorithm}\label{alg1}
\textbf{Input}: The regularization parameter $\lambda > 0$ and stepsizes $\{\eta_t\}_{t=1}^T$.\\
\textbf{Initialization}: Set $T^+_0=T^-_0=0$, $\c_0^+=\c_0^-=\bf{0}$, $\w_0=\bf{0}$ and $\Gamma_0^+ = \Gamma_0^- = [\mathbf{0}]_{d\times u}$ for some $u > 0$

\begin{algorithmic}[1]
\FOR{$t=1,2,\ldots,T$}
\STATE  Receive a training example $(\x_t,y_t)$
\IF{$y_t=+1$}
\STATE $T_t^+=T^+_{t-1}+1$ and $T_t^-=T^-_{t-1}$;
\STATE $\c_t^+=\c_{t-1}^++\frac{1}{T_t^+}(\x_t-\c_{t-1}^+)$ and $\c_t^-=\c_{t-1}^-$;
\STATE Update $\Gamma_t^+$ and $\Gamma_t^-=\Gamma_{t-1}^-$;
\STATE Calculate the gradient $\gh_t(\w_{t-1})$
\ELSE
\STATE $T_t^-=T_{t-1}^-+1$ and $T_t^+=T_{t-1}^+$;
\STATE $\c_{t}^-=\c_{t-1}^-+\frac{1}{T_t^-}(\x_t-\c_{t-1}^-)$ and $\c_{t}^+=\c_{t-1}^+$;
\STATE Update $\Gamma_{t}^-$ and $\Gamma_{t}^+=\Gamma_{t-1}^+$;
\STATE Calculate the gradient $\gh_t(\w_{t-1})$
\ENDIF
\STATE $\w_{t}=\w_{t-1} - \eta_t\gh_t(\w_{t-1})$
\ENDFOR
\end{algorithmic}
\end{algorithm}

Algorithm~\ref{alg1} highlights the key steps of the proposed algorithm. We initialize $\Gamma_0^-= \Gamma_0^+= [\mathbf{0}]_{d\times d}$, where $u = d$. At each iteration, we set $\Gamma_t^+=S_t^+$ and $\Gamma_t^-=S_t^-$, and update $\Gamma_t^+$ (Line 6) and $\Gamma_t^-$ (Line 11), respectively, by using the following equations
\begin{eqnarray*}
& \Gamma_t^+ = \Gamma_{t-1}^+ + \frac{\x_t\x_t^\top-\Gamma^+_{t-1}}{T_t^+} + \c_{t-1}^+[\c_{t-1}^+]^{\top} - \c_{t}^+[\c_{t}^+]^{\top},&  \\
& \Gamma_t^- = \Gamma_{t-1}^- + \frac{\x_t\x_t^\top-\Gamma^-_{t-1}}{T_t^-} + \c_{t-1}^-[\c_{t-1}^-]^{\top} - \c_{t}^-[\c_{t}^-]^{\top}.&
\end{eqnarray*}
Finally, the stochastic gradient $\gh_t(\w_{t-1})$ of Lines~7 and 12 in Algorithm~\ref{alg1} are given by $\nabla \mathcal{L}_t(\w_{t-1})$ that are calculated by Eqs.~\eqref{eqn:grad-minus-ls} and \eqref{eqn:grad-plus-ls}, respectively.

\paragraph{Dealing with High-Dimensional Data.} One limitation of the approach in Algorithm~\ref{alg1} is that the storage cost of the two covariance matrices $S_t^+$ and $S_t^-$ is $O(d^2)$, making it unsuitable for high-dimensional data. We tackle this by developing a randomized algorithm that approximates the covariance matrices by low-rank matrices. We are motivated by the observation that $S_t^+$ and $S_t^-$ can be written, respectively, as
\begin{eqnarray*}
&S_t^+  =  \frac{1}{T_t^+} \left(X_t^+ - \c^+_t\mathbf{1}_{T_t^+}^{\top}\right)I_{T_t^+} \left(X_t^+ - \c_t^+\mathbf{1}_{T_t^+}\right)^{\top},& \\
&S_t^-  =  \frac{1}{T_t^-} \left(X_t^- - \c^-_t\mathbf{1}_{T_t^-}^{\top}\right)I_{T_t^-} \left(X_t^- - \c_t^- \mathbf{1}_{T_t^-}\right)^{\top}, &
\end{eqnarray*}
where $I_t$ is an identity matrix of size $t\times t$ and $\mathbf{1}_t$ is an all-one vector of size $t$. To approximate $S_t^+$ and $S_t^-$, we approximate the identify matrix $I_t$ by a matrix of rank $\tau \ll d$. To this end, we randomly sample $\r_i \in \R^\tau, i=1,\ldots, t$ from a Gaussian distribution $\mathcal{N}(0, I_\tau)$, and approximate $I_t$ by $R_tR^{\top}_t$, where $R_t = \frac{1}{\tau}(\r_1, \ldots, \r_t)^{\top} \in \R^{t\times \tau}$. We further divide $R_t$ into two matrices where $R_t^+ \in \R^{T_t^+\times \tau}$ and $R_t^- \in \R^{T_t^-\times \tau}$ that contain the subset of the rows in $R_t$ corresponding to all the positive and negative instances received before the $t$-th iteration, respectively. Therefore, the covariance matrices $S_t^+$ and $S_t^-$ can be approximated, respectively, by
\[
\Sh_t^+  =  \frac{1}{T_t^+} Z_t^+ [Z_t^+]^{\top} - \ch_{t-1}^+ [\ch_{t-1}^+]^{\top} \quad\text{ and }\quad \Sh_t^-  =  \frac{1}{T_t^-}Z_t^- [Z_t^-]^{\top} - \ch_{t-1}^- [\ch_{t-1}^-]^{\top}
\]
where
\begin{eqnarray*}
&&Z_t^+ = X_t^+R^+_t, \quad\ch_t^+ = {\c_t^+} \mathbf{1}_{T_t^+}^{\top} R^{+}_t/{T_t^+}\\
&&Z_t^- = X_t^{-} R^-_t, \quad \ch_t^- = {\c_t^-} \mathbf{1}_{T_t^-}^{\top} R^{-}_t/{T_t^-}.
\end{eqnarray*}
Based on approximate covariance matrix $\Sh_t^\pm$, the approximation algorithm essentially tries to minimize $\sum_{t=1}^T \Lh_t(\w)$, where
\begin{equation}\label{eq:t3}
\Lh_t(\w) = \w^{\top}(\c_{t-1}^- - \x_t) + \frac{1}{2}(1 + \w^{\top}\Sh_t^-\w) + \frac{\lambda}{2}|\w|^2 +\frac{1}{2}\w^{\top}(\x_t - \c_{t-1}^-)(\x_t - \c_{t-1}^{-})^{\top}\w
\end{equation}
if $y_t = 1$; otherwise,
\begin{equation}\label{eq:t4}
\Lh_t(\w) =  \w^{\top}(\x_t - \c_{t-1}^+) + \frac{1}{2}(1 + \w^{\top}\Sh_t^+\w) +\frac{\lambda}{2}|\w|^2 +\frac{1}{2}\w^{\top}(\x_t - \c_{t-1}^+)(\x_t - \c_{t-1}^+)^{\top}\w.
\end{equation}
Further, we have the following recursive formulas:
\begin{eqnarray}
Z_t^+ &=& Z_{t-1}^++\x_t\r_t^{\top}\mathbb{I}[y_t=+1]/\sqrt{m},\label{eq:t1}\\
Z_t^- &=& Z_{t-1}^- + \x_t\r_t^{\top}\mathbb{I}[y_t=-1]/\sqrt{m}.\label{eq:t2}
\end{eqnarray}
It is important to notice that we do not need to calculate and store the approximate covariance matrices $\Sh_t^{+}$ and $\Sh_t^{-}$ explicitly. Instead, we only need to maintain matrices $Z_t^{+}$ and $Z_t^{-}$ in memory. This is because the stochastic gradient $\gh_t(\w)$ based on the approximate covariance matrices can be computed directly from $Z_t^{+}$ and $Z_t^{-}$. More specifically, $\gh_t(\w)$ is computed as
\begin{equation}\label{eqn:grad-plus}
\gh_t(\w)=\c_{t-1}^--\x_t+\lambda\w+(\x_t - \c_{t-1}^-)(\x_t-\c_{t-1}^-)^{\top}\w
+\left(Z_t^- [Z_t^-]^{\top}/{T_t^-}-\ch_{t-1}^-[\ch_{t-1}^-]^{\top}\right)\w
\end{equation}
for $y_t = 1$; otherwise
\begin{equation}\label{eqn:grad-minus}
\gh_t(\w)=\x_t - \c_{t-1}^+ +\lambda\w + (\x_t - \c_{t-1}^+)(\x_t - \c_{t-1}^+)^{\top} \w + \left({Z_t^+ [Z_t^+]^{\top}}/ {T_t^+} - \ch_{t-1}^+[\ch_{t-1}^+]^{\top}\right)\w.
\end{equation}
We require a memory of $O(\tau d)$ instead of $O(d^2)$ to calculate $\gh_t(\w)$ by using the trick $A[A]^\top \w=A([A]^\top \w)$, where $A\in\mathbb{R}^{d\times1}$ or $\mathbb{R}^{d\times \tau}$.

To implement the approximate approach, we initialize $\Gamma_0^-= \Gamma_0^+= {\bf[0]}_{d\times \tau}$ in Algorithm~\ref{alg1}, where $u = \tau$. At each iteration, we set $\Gamma_t^+=Z_t^+$ and $\Gamma_t^-=Z_t^-$ , and compute the gradient $\gh_t(\w_{t-1})$ of Lines~7 and 12 in Algorithm~\ref{alg1} by Eqs.~\eqref{eqn:grad-plus} and \eqref{eqn:grad-minus}, respectively. $\Gamma_t^{+}$ and $\Gamma_t^{-}$ are updated by Eqs.~\eqref{eq:t1} and \eqref{eq:t2}, respectively.

\paragraph{Remark.} An alternative approach for the high-dimensional case is through the random projection~\citep{Johnstone2007,Hsu:Kakade:Zhang2012}. Let $H \in \R^{d\times \tau}$ be a random Gaussian matrix, where $\tau \ll d$. By performing random projection using $H$, we compute a low-dimensional representation for each instance $\x_t$ as $\widehat{\x}_t = H^{\top}\x_t \in \R^\tau$ and will only maintain covariance matrices of size $\tau\times \tau$ in memory. Despite that it is computationally attractive, this approach performs significantly worse than the randomized low-rank approximation algorithm, according to our empirical study. This may owe to the fact that the random projection approach is equivalent to approximating $S_t^{\pm} = I_d S_t^{\pm} I_d$ by $HH^{\top}S_t^{\pm}HH^{\top}$, which replaces both the left and right identity matrices of $S_t^{\pm}$ with $HH^{\top}$. In contrast, our proposed approach only approximates one identity matrix in $S_t^{\pm}$, making it more reliable for tackling high-dimensional data.

\section{Main Theoretical Result}\label{sec:theory}
In this section, we present the main theoretical results for our proposed algorithm. The following theorem shows the consistency of square loss, and the detailed proof is deferred in Section~\ref{pf:thm:con}.
\begin{theorem}\label{thm:con}
For square loss $\ell(t)=(1-t)^2$, the surrogate loss $\Psi(f,{\bm x},{\bm x}')= \ell(f({\bm x})- f({\bm x}'))$ is consistent with AUC.
\end{theorem}

\

Define $\w_*$ as
\[
\w_* = \mathop{\arg\min}_{\w} \sum_{t} \L_t(\w).
\]
We are in the position to present the following convergence rate for Algorithm~\ref{alg1} when the full covariance matrices are provided, and the detailed proof is deferred in Section~\ref{pf:thm:SGD}
\begin{theorem}\label{thm:SGD}
For $\|\x_t\|\leq 1$ $(t \in [T])$, $\|\w_*\| \leq B$ and $T L^* \geq \sum_{t=1}^T \L_t(\w_*)$, we have
\[
\sum\nolimits_{t}\mathcal{L}_t(\w_t)- \sum\nolimits_{t} \mathcal{L}_t(\w_*)\leq {2\kappa B^2} + {B\sqrt{2\kappa TL^*}},
\]
where $\kappa=4+\lambda$ and $\eta_t={1}/(\kappa+\sqrt{(\kappa^2+\kappa TL^*/B^2})$.
\end{theorem}

This theorem presents an $O(1/T)$ convergence rate for the OPAUC algorithm if the distribution is separable, i.e., $L^*=0$, and an $O(1/\sqrt{T})$ convergence rate for general case. Compared to the online AUC optimization algorithm~\citep{Zhao:Hoi:Jin:Yang2011}, which achieves at most $O(1/\sqrt{T})$ convergence rate, our proposed algorithm clearly reduce the regret. The faster convergence rate of our proposed algorithm owes to the smoothness of the square loss, an important property that has been explored by some studies of online learning \citep{Rakhlin:Shamir:Sridharan2012} and generalization error bound analysis~\citep{Srebro:Sridharan:Tewari2010}.

{\bf Remark}: The bound in Theorem~\ref{thm:SGD} does not explicitly explore the strongly convexity of $\L_t(\w)$, which can lead to an $O(1/T)$ convergence rate. Instead, we focus on exploiting the smoothness of the loss function, since we did not introduce a bounded domain for $\w$. Due to the regularizer $\lambda|\w|^2/2$, we have $|\w_*|\leq1/\lambda$, and it is reasonable to restrict $\w_t$ by $|\w_t| \leq 1/\lambda$, leading to a regret bound of $O(\ln T /[\lambda^3 T])$ by applying the standard stochastic gradient descent with $\eta_t = 1/[\lambda t]$. This bound is preferred only when $\lambda = \Omega(T^{-1/6})$, a scenario which rarely occurs in empirical study. This problem may also be addressable by exploiting the epoch gradient method \citep{Nocedal:Wright1999}, a subject of our future study.

\

We now consider the case when covariance matrices are approximated by low-rank matrices. Note that the low-rank approximation is accurate only if the eigenvalues of covariance matrices follow a skewed distribution. To capture the skewed eigenvalue distribution, we introduce the concept of \emph{effective numerical rank} \citep{Hansen1987} that generalizes the rank of matrix:
\begin{definition}
For a positive constant $\mu>0$ and semi-positive definite matrix $M\in\R^{d\times d}$ of eigenvalues $\{\nu_i\}$, the effective numerical rank w.r.t. $\mu$ is defined to be $r(M, \mu)=\sum_{i=1}^d{\nu_i}/(\mu+\nu_i)$.
\end{definition}
It is evident that the effective numerical rank is upper bounded by the true rank, i.e., $r(M, \mu) \leq rank(M)$. To further see how the concept of effective numerical rank captures the skewed eigenvalue distribution, consider a PSD matrix $M$ of full rank with $\sum_{i=k}^d \nu_i \leq \mu$ for small $k$. It is easy to verify that $r(M, \mu) \leq k$, i.e., $M$ can be well approximated by a matrix of rank $k$.

Define the effective numerical rank for a set of matrices $\{M_t\}_{t=1}^T$ as \[
r\left(\{M_t\}_{t=1}^T, \mu\right) = \max_{1 \leq t \leq T} r(M_t, \mu).
\]
Under the assumption that the effective numerical rank for the set of covariance matrices $\{S^{\pm}_t\}_{t=1}^T$ is small (i.e., $S_t^{\pm}$ can be well approximated by low-rank matrices), the following theorem gives the convergence rate for $|\sum_t\Lh_t(\w_t) - \sum_t\L_t(\w_*)|$, where $\Lh_t(\w_t)$ are given by Eqs.~\eqref{eq:t3} and \eqref{eq:t4}.
\begin{theorem}\label{thm:main}
Let $r = r(\{S^{\pm}_t\}_{t=1}^T, \lambda)$ be the effective numerical rank for the sequence of covariance matrices $\{S^{\pm}_{t}\}_{t=1}^T$. For $0<\delta<1$, $0<\epsilon\leq1/2$, $|\w_*| \leq B$, $\|\x_t\| \leq 1$ $(t \in [T])$ and $TL^* \geq \sum_{t=1}^T \L_t(\w_*)$, we have with probability at least $1 - \delta$,
\[
\left|\sum\nolimits_{t}\big({\Lh_t(\w_t)}- \mathcal{L}_t(\w_*)\big)\right|\leq 2\epsilon TL^*+{2\kappa B^2 + B\sqrt{2\kappa TL^*}}
\]
provided $\tau \geq \frac{32r \lambda}{\epsilon^2}\log\frac{2d T}{\delta}$, where $\kappa=4+\lambda$ and $\eta_t={1}/(\kappa+\sqrt{(\kappa^2+\kappa TL^*/B^2})$.
\end{theorem}

The detailed proof is presented in Section~\ref{pf:thm:main}. For the separable distribution $L^*=0$, we also obtain an $O(1/T)$ convergence rate when the covariance matrices are approximated by low-rank matrices. Compared with Theorem~\ref{thm:SGD}, Theorem~\ref{thm:main} introduces an additional term $2\epsilon L^*$ in the bound when using the approximate covariance matrices, and it is noteworthy that the approximation does not significantly increase the bound of Theorem~\ref{thm:SGD} if $2\epsilon TL^* \leq B\sqrt{2(4+\lambda) TL^*}$, i.e., $\epsilon\leq B\sqrt{2(\lambda+4) /TL^*}$. This implies that the approximate algorithm will achieve similar performance as the one using the full covariance matrices provided $\tau = \Omega(r \lambda T (\log d+ \log T)/(\lambda+4))$. When $\lambda = O(1/T)$, this requirement is reduced to $\tau = \Omega(r[\log d + \log T])$, a logarithmic dependence on dimension $d$.

\section{Proofs}\label{sec:pf}
In this section, we present detailed proofs for our main theorems.

\subsection{Proof of Theorem~\ref{thm:con}}\label{pf:thm:con}
Let $\mathcal{X}=\{\x_1,\x_2,\ldots,\x_n\}$ with instance-marginal probability $p_i=\Pr[x_i]$ and conditional probability $\xi_i=\Pr[y=+1|{\x_i}]$,  and we denote by the expected risk
\[
R_\Psi(f)=C_0+\sum_{i\neq j}p_ip_j \left(\xi_i(1-\xi_j) \ell(f(\x_i)-f(\x_j))+\xi_j(1-\xi_i)\ell(f(\x_j)-f(\x_i))\right)
\]
where $\ell(t)=(1-t)^2$ and $C_0$ is a constant with respect to $f(\x_i)$ ($1\leq i\leq n$). According to the analysis of \citep{Gao:Zhou2012}, it suffices to prove that, for every optimal solution $f$, i.e., $R_\Psi(f)=\inf_{f'}R_\Psi(f')$, we have $f(\x_i)>f(\x_j)$ if $\xi_i>\xi_j$.

If $\mathcal{X}=\{{\x}_1,{\x}_2\}$, then minimizing $R_\phi(f)$ gives the optimal solution $f=(f({\x}_1),f({\x}_2))$ such that
\[
f({\x}_1)-f({\x}_2)=\sgn(\xi({\x}_1)-\xi({\x}_2)) \text{ for } \xi({\x}_1)\neq\xi({\x}_2),
\]
which shows the consistency of least square loss.

For $\mathcal{X}=\{{\x}_1,{\x}_2,\cdots,{\x}_n\}$ with $n\geq3$, if $\xi_i(1-\xi_i)=0$ for every $1\leq i\leq n$, then minimizing $R_\Psi(f)$ gives the optimal solution $f=(f_1,f_2,\cdots,f_n)$ such that
\[
f_j=f_i+1 \text{ for every }\xi_i=1 \text{ and }\xi_j=-1,
\]
which shows the consistency of least square loss.

If $\mathcal{X}=\{{\x}_1,{\x}_2,\cdots,{\x}_n\}$ with $n\geq3$, and there exists some $i_0$ s.t. $\xi_{i_0}(1-\xi_{i_0})\neq 0$, then the subgradient conditions give optimal solution such that
\[
\sum_{k\neq i}p_k(\xi_i+\xi_k-2\xi_i\xi_k)(f(\x_i)-f(\x_k))=\sum_{k\neq i}p_k(\xi_i-\xi_k)\text{ for each }1\leq i\leq n.
\]
Solving the above $n$ linear equations, we obtain the optimal solution $f=(f_1,f_2,\ldots,f_n)$, i.e.,
\[
f(\x_i)-f(\x_j)=(\xi_i-\xi_j)\frac{\prod_{k\neq i,j} \sum_{l=1}^n p_l(\xi_l+\xi_k-2\xi_l\xi_k)}{\sum_{s_i\geq0 \atop s_1+\cdots+s_n=n-2} p^{s_1}_1\cdots p^{s_n}_n \Gamma(s_1,s_2,\cdots,s_n)}
\]
where $\Gamma$ is a polynomial in $\xi[k_1]+\xi[k_2]-2\xi[k_1]\xi[k_2]$ for $1\leq k_1,k_2\leq n$. In the following, we will derive the specific expression for $\Gamma(s_1,s_2,\cdots,s_n)$. Denote by $\mathcal{A}=\{i\colon s_i\geq1\}$ and $\mathcal{B}=\{i\colon s_i=0\}=\{b_1,b_2,\cdots,b_{|\mathcal{B}|}\}$.
\begin{itemize}
\item If $|\mathcal{A}|=1$, i.e., $\mathcal{A}=\{i_1\}$ for some $1\leq i_1\leq n$, then
\[
\Gamma(s_1,s_2,\cdots,s_n)=\prod_{k\in \mathcal{B}} (\xi_{i_1}+\xi_{k}-2\xi_{i_1}\xi_{k}).
\]
\item If $|\mathcal{A}|=2$, i.e., $\mathcal{A}=\{i_1,i_2\}$ for some $1\leq i_1,i_2\leq n$, then we denote by
\[
\mathcal{A}_1= \{s_{i_1}\odot i_1\} \bigcup \{s_{i_2}\odot i_2\}
\]
where $\{s_{i_k}\odot {i_k}\}$ denotes the multi-set $\{i_k,i_k, \ldots, i_k\}$ of size $s_{i_k}$ for $k=1,2$. It is clear that $|\mathcal{B}|=|\mathcal{A}_1|=n-2$. Further, we denote by $\mathcal{G}(\mathcal{A}_1)$ the set of all permutations of $\mathcal{A}_1$. Therefore, we have
\[
\Gamma(s_1,s_2,\cdots,s_n)=(\xi_{i_1}+\xi_{i_2}-2\xi_{i_1}\xi_{i_2}) \sum_{\pi=\pi_1\cdots\pi_{n-2}\in\mathcal{G}(\mathcal{A}_1)} \prod_{k=1}^{n-2} (\xi_{\pi_i}+\xi_{b_i}-2\xi_{\pi_i}\xi_{b_i}).
\]
\item If $|\mathcal{A}|>2$, then, for $i_1\neq i_2\in\mathcal{A}$, we denote by the multi-set
\[
\mathcal{A}_1(i_1,i_2)= \{s_{i_1}\odot i_1\} \bigcup \{s_{i_2}\odot i_2\} \bigcup \left( \bigcup_{k\in \mathcal{A}\setminus\{i_1,i_2\}} \{(s_k-1)\odot k\} \right),
\]
and it is easy to derive $|\mathcal{A}_1|=|\mathcal{B}|$. Further, we denote by $\mathcal{G}(\mathcal{A} \setminus \{i_1,i_2\})$ and $\mathcal{G}(\mathcal{A}_1)$ the set of all permutations of $\mathcal{A}\setminus\{i_1,i_2\}$ and $\mathcal{A}_1$, respectively. Therefore, we set
\begin{multline*}
\Gamma_1(i_1,i_2,\mathcal{A})=\sum_{\pi=\pi_1\pi_2\cdots\pi_{|\mathcal{A}|-2} \in \mathcal{G}(\mathcal{A} \setminus\{i_1,i_2\})} (\xi_{i_1}+\xi_{\pi_1}-2\xi_{i_1} \xi_{\pi_1})(\xi_{\pi_1}+\xi_{\pi_2}-2\xi_{\pi_1}\xi_{\pi_2}) \\
\times \cdots \times(\xi_{\pi_{|A|-3}}+\xi_{\pi_{|A|-2}}- 2\xi_{\pi_{|A|-3}} \xi_{\pi_{|A|-2}}) (\xi_{\pi_{|A|-2}}+\xi_{i_2}-2\xi_{i_2}\xi_{\pi_{|A|-2}}),
\end{multline*}
and we have
\[
\Gamma(s_1,s_2,\cdots,s_n)=\sum_{i_1\neq i_2 \atop i_1,i_2\in\mathcal{A}} \Gamma_1(i_1,i_2,\mathcal{A})\sum_{\pi=\pi_1\pi_2\ldots \pi_{|\mathcal{B}|} \in \mathcal{G}(\mathcal{A}_1)} \prod_{k=1}^{|B|} (\xi_{\pi_k}+ \xi_{b_k}- 2\xi_{\pi_k}\xi_{b_k})
\]
where $\mathcal{B}=\{b_1,b_2,\ldots,b_{|\mathcal{B}|}\}$.
\end{itemize}
Since there exist some $i_0$ s.t. $\xi_{i_0} (1-\xi_{i_0}) \neq 0$, we have
\[
\frac{\prod_{k\neq i,j} \sum_{l=1}^n p_l(\xi_l+\xi_k-2\xi_l\xi_k)}{\sum_{s_i\geq0 \atop s_1+\cdots+s_n=n-2} p^{s_1}_1\cdots p^{s_n}_n \Gamma(s_1,s_2,\cdots,s_n)}>0.
\]
Therefore, it is evident that $f(\x_i)>f(\x_j)$ if $\xi_i>\xi_j$, and this theorem follows as desired.\qed

\subsection{Proof of Theorem~\ref{thm:SGD}}\label{pf:thm:SGD}
This proof is motivated from \citep{Shwartz2007,Srebro:Sridharan:Tewari2010}. Recall
\[
\mathcal{L}_t(\w)= \frac{\lambda}{2}|\w|^2 + \frac{\sum_{i=1}^{t-1}\mathbb{I}[y_i\neq y_t] (1-y_t(\x_t-\x_i)^\top\w)^2}{2|\{i\in[t-1]:y_i\neq y_t\}|}.
\]
For $|\w_*|\leq B$ and convex $\mathcal{L}_t(\w)$, we have
\begin{equation}\label{eq:tmp1}
\mathcal{L}_t(\w_t)- \mathcal{L}_t(\w_*)\leq \nabla\mathcal{L}_t(\w_t)^\top (\w_t-\w_*).
\end{equation}
It is easy to derive that $\nabla \mathcal{L}_t(\w_t)$ equals to
\[
\lambda\w_t- \frac{\sum_{i=1}^{t-1}\mathbb{I}[y_i\neq y_t] (1-y_t(\x_t-\x_i)^\top\w_t)y_t(\x_t-\x_i)}{|\{i\in[t-1]:y_i\neq y_t\}|},
\]
and therefore, for any $\w$ and $|\x_i|\leq1$
\[
|\nabla \mathcal{L}_t(\w_t) -\nabla \mathcal{L}_t(\w)|\leq (4+\lambda)|\w_t-\w|.
\]
Denote by
\[
\w_{t*}=\mathop{\arg\min}\limits_{\w} \mathcal{L}_t(\w_t),
\]
which implies that $\nabla \mathcal{L}_t(\w_{t*})=0$ for convex and smooth $\mathcal{L}_t$. Based on \citep[Theorem~2.1.5]{Nesterov2003}, we have
\begin{equation}\label{eq:tmp2}
{|\nabla \mathcal{L}_t(\w_t)|^2}=|\nabla \mathcal{L}_t(\w_t)- \nabla\mathcal{L}_t(\w_{t*})|^2 \leq2(\lambda +4)L_t(w_t)
\end{equation}
where the inequality holds from $\mathcal{L}_t(\w_{t*})\geq0$ and $\nabla \mathcal{L}_t(\w_{t*})=0$. Moreover, we have
\[
|\w_{t+1}-\w_*|^2=|\w_{t}-\eta_t \nabla \mathcal{L}_t(\w_t)-\w_*|^2
=|\w_{t}-\w_*|^2 - 2\eta_t \nabla\mathcal{L}_t(\w_t)^\top (\w_t-\w_*) + \eta_t^2 |\nabla \mathcal{L}_t(\w_t)|^2,
\]
and this yields that, by using Eqs.~\eqref{eq:tmp1} and \eqref{eq:tmp2},
\[
(1-(4+\lambda)\eta_t)\mathcal{L}_t(\w_t)- \mathcal{L}_t(\w_*) \leq\frac{1}{2\eta_t}|\w_{t}-\w_*|^2-\frac{1}{2\eta_t}|\w_{t+1}-\w_*|^2.
\]
Summing over $t=0,\ldots,T-1$ and rearranging, we obtain
\begin{eqnarray*}
\lefteqn{\sum_{t=0}^{T-1}(1-(4+\lambda)\eta_t)\mathcal{L}_t(\w_t)- \sum_{t=0}^{T-1}\mathcal{L}_t(\w_*)} \\
&\leq& \frac{1}{2\eta_0} |\w_0-\w_*|^2-\frac{1}{2\eta_{T-1}} |\w_{T}-\w_*|^2 +\sum_{i=1}^{T-2}(\frac{1}{2\eta_{i+1}}-\frac{1}{2\eta_{i}})|\w_i-\w_*|^2.
\end{eqnarray*}
By setting $\eta_t=\eta$, we have
\[
\frac{1}{2\eta_0} |\w_0-\w_*|^2 -\frac{1}{2\eta_{T-1}} |\w_{T}-\w_*|^2
\leq \frac{1}{2\eta} |\w_*|^2\leq \frac{B^2}{2\eta}
\]
from $\w_0={\bf 0}$ and $|\w_*|\leq B$, and we further get
\begin{eqnarray*}
\sum_{t=0}^{T-1}\mathcal{L}_t(\w_t)- \sum_{t=0}^{T-1} \mathcal{L}_t(\w_*)
&\leq& \frac{1}{1-(4+\lambda)\eta}\left(\frac{B^2}{2\eta}+ (4+\lambda)\eta\sum_{t=0}^{T-1} \mathcal{L}_t(\w_*)\right)\\
&\leq& \frac{1}{1-(4+\lambda)\eta}\left(\frac{B^2}{2\eta}+(4+\lambda)\eta TL^*\right).
\end{eqnarray*}
This theorem holds by putting
\[
\eta=\frac{1}{4+\lambda+\sqrt{(4+\lambda)^2+(4+\lambda)TL^*/B^2}}
\]
into the above formula and using the formula $\sqrt{a+b}\leq \sqrt{a}+\sqrt{b}$.\qed

\subsection{Proof of Theorem~\ref{thm:main}}\label{pf:thm:main}
Before the detailed proof of Theorem~\ref{thm:main}, we begin with some useful results:
\begin{lemma}\label{lem:tmp}
Let $S_1=\text{diag}(s_{1i})$ and $S_2=\text{diag}(s_{2i})$ be two $d\times d$ diagonal matrices such that $s_{1i}\neq0$ and $s_{1i}^2+s_{2i}^2=1$ for all $i$. For a Gaussian random matrix $R\in\R^{d\times \tau}$, we set $Z=S_1S_1+S_2RR^\top S_2$ and $r=\sum_i s_{2i}^2$, and the followings hold
\[
\Pr[\lambda_1(Z)\geq 1+\epsilon] \leq d\exp(-\tau\epsilon^2/32r) \quad\text{ and }\quad \Pr[\lambda_p(Z)\leq 1-\epsilon] \leq d\exp(-\tau\epsilon^2/32r),
\]
where $\lambda_k(Z)$ denotes the $k$-th largest eigenvalue of matrix $Z$.
\end{lemma}
\begin{proof} This proof technique is motivated from \citep{gittens-2011-tail} by adding a bias matrix. Let $g(\theta)=\frac{\theta^2}{2-2\theta}$. Then, we have
\begin{eqnarray*}
\lefteqn{\Pr[\lambda_1(M)\geq1+\epsilon]}\\
&\leq&\inf_{\theta>0}\text{tr}\exp\left\{\theta\left(S_1S_1+E[S_2RR^\top S_2]-(1+\epsilon)I\right)+g(\theta)E[(S_2RR^\top S_2)^2]/\tau\right\}\\
&\leq&\inf_{\theta>0}\text{tr}\exp\{-\theta\epsilon+8g(\theta)\text{tr}(S_2^2)S_2^2\}\\
&\leq&\inf_{\theta>0}d\exp\{-\theta\epsilon+8rg(\theta)\}\leq d\exp(-\tau\epsilon^2/32r).
\end{eqnarray*}
In a similar manner,
\begin{eqnarray*}
\lefteqn{\Pr[\lambda_p(M)\leq1-\epsilon]}\\
&\leq&\inf_{\theta>0}\text{tr}\exp\left\{\theta \big(S_1S_1+E[S_2RR^\top S_2]-(1-\epsilon)I\big)+g(\theta)E[(S_2RR^\top S_2)^2]/\tau\right\}\\
&\leq&\inf_{\theta>0}d\exp\{-\theta\epsilon+8rg(\theta)\}\leq d\exp(-\tau\epsilon^2/32r).
\end{eqnarray*}
This completes the proof.
\end{proof}

Let $M \in \R^{d\times d}$ be a positive semi-definite (PSD) matrix with effective numerical rank $r(M, \mu)$ for $\mu > 0$. We define two matrices $K$ and $\SK$, respectively, as
\[
K :=\mu I_d+ M \text{ and }\SK:=\mu I_d + M^{-1/2}RR^\top M^{-1/2},
\]
where $R\in\R^{d\times m}$ is a (Gaussian) random matrix. Based on Lemma~\ref{lem:tmp}, we have the following theorem that bounds the difference between $K - \SK$:
\begin{lemma}\label{thm:tmp}
Let  $r(M, \mu)$ be the numerical rank for $\mu>0$ and PSD matrix $M$. Then, for $\delta>0$ and $\epsilon>0$, the following holds with probability at least $1-\delta$
\[
\|I-K^{-1/2}\SK K^{-1/2}\|_2 \leq \epsilon,
\]
where $\|Z\|_2$ measures the spectral norm of matrix $Z$, provided
\[
\tau \geq \frac{32r(M, \mu)}{\epsilon^2}\log(2d/\delta).
\]
\end{lemma}

\begin{proof} Let $M=U\text{diag}(\sigma^2_i)V^\top$ be the singular value decomposition of $M$. We define
\[
S_1=\mbox{diag}\left(\sqrt{\frac{\mu}{\sigma_1^2 + \mu}}, \ldots, \sqrt{\frac{\mu}{\sigma_d^2 + \mu}}\right) \quad\text{ and } \quad S_2= \mbox{diag}\left(\frac{\sigma_1}{\sqrt{\sigma_1^2 + \mu}}, \ldots, \frac{\sigma_d}{\sqrt{\sigma_d^2 + \mu}}\right).
\]
It is easy to observe that
\[
Z=K^{-1/2}\SK K^{-1/2}=U(S_1^2 + S_2V^{\top}RR^{\top}V)U^{\top}= U(S_1^2 + S_2\widehat{R}\widehat{R}^{\top})U^{\top}
\]
where $\widehat{R} = V^{\top}R \in \R^{d\times \tau}$ is a also Gaussian random matrix because $V$ is an orthonormal matrix. Parameter $r$ in Lemma~\ref{lem:tmp} is given by
\[
r = \sum_{i=1}^d s^2_{2,i} = \sum_{i=1}^d \frac{\sigma_i^2}{\sigma_i^2 + \mu} = r(M, \mu).
\]
Using Lemma~\ref{lem:tmp}, the followings hold with a probability at least $1 - \delta$,
\[
\lambda_{\max}\left(Z\right)=\|K^{-1/2}\SK K^{-1/2}\|_2 \leq 1 + \epsilon \quad \text{ and } \quad \lambda_{\min}(Z) = \lambda_d\left(K^{-1/2}\SK K^{-1/2}\right) \geq 1 - \epsilon,
\]
which yields that $\|Z - I\| \leq \epsilon$ provided
\[
\tau \geq \frac{32 r}{\epsilon^2}\log\frac{2d}{\delta}.
\]
This lemma follows as desired.
\end{proof}

Recall that
\[
\Lh_t(\w) = \frac{\lambda}{2}|\w|^2 + \w^{\top}(\c_{t-1}^- - \x_t) + \frac{1}{2} + \frac{1}{2}\left(\w^{\top}(\x_t - \c_{t-1}^-)(\x_t - \c_{t-1}^{-})^{\top}\w +\w^{\top}\Sh_t^-\w\right)
\]
if $y_t = 1$; otherwise,
\[
\Lh_t(\w) = \frac{\lambda}{2}|\w|^2 + \w^{\top}(\x_t - \c_{t-1}^+) + \frac{1}{2}
+\frac{1}{2}\left(\w^{\top}(\x_t - \c_{t-1}^+)(\x_t - \c_{t-1}^+)^{\top}\w+\w^{\top}\Sh_t^+\w\right).
\]
We further define $\wh_*$ as the optimal solution that minimizes the loss based on approximate covariance matrices, i.e.
\begin{eqnarray*}
\wh_* = \mathop{\arg\min}\limits_{\w} \sum_{t=1}^T \Lh_t(\w).
\end{eqnarray*}

Based on Lemma~\ref{thm:tmp}, the following theorem gives an upper bound for $|\sum_t\Lh_t(\wh_*) - \sum_t\L_t(\w_*)|$.
\begin{theorem}\label{thm:random}
Let $r(\{S_t^{\pm}\}_{t=1}^T, \lambda)$ be the effective numerical rank for the set of covariance matrices $S_t^{\pm}, t=1, \ldots, T$ with respect to the regularization parameter $\lambda$. Then, for any $0<\delta$ and $0<\epsilon\leq 1/2$, the followings hold with probability at least $1-\delta$
\begin{eqnarray}
&|\wh_* - \w_*|\leq 2\epsilon|\w_*|& \label{eq:temp3}\\
&\left|\sum_t\Lh_t(\wh_*) - \sum_t\L_t(\w_*)\right|\leq 2\epsilon\sum_t\L_t(\w_*)& \label{eq:temp4}
\end{eqnarray}
provided that
\[
\tau \geq \frac{32r(S, \lambda)}{\epsilon^2}\log\frac{2d}{\delta T}
\]
\end{theorem}
\begin{proof}
We first rewrite $\L(\w) = \sum_{t=1}^T \L_t(\w)$ as
\[
\L(\w) = \frac{1}{2} + \w^{\top}\a + \frac{1}{2}\w^{\top}\left(A_1 + A_2\right)\w
\]
where
\begin{eqnarray*}
\a & = & \sum_{t=1}^T \frac{\mathbb{I}[y_t = 1]\left(\c_{t-1}^- - \x_t\right) + \mathbb{I}[y_t = -1]\left( \c_{t-1}^+ - \x_t\right)}{T} \\
A_1 & = & \lambda I_d +\frac{1}{T} \sum_{t=1}^T \left(\mathbb{I}[y_t = 1] S_{t-1}^- + \mathbb{I}[y_t = -1] S_{t-1}^+\right) \\
A_2 & = & \frac{1}{T} \sum_{t=1}^T \mathbb{I}[y_t = 1](\x_t - \c_t^-)(\x_t - \c_t^-)^{\top} + \frac{1}{T}\sum_{t=1}^T \mathbb{I}[y_t = -1](\x_t - \c_t^+)(\x_t - \c_t^+)^{\top}.
\end{eqnarray*}
Similarly, we rewrite $\Lh(\w) = \sum_{t=1}^T \Lh_t(\w)$ as
\[
\Lh(\w) = \frac{1}{2} + \w^{\top}\a + \frac{1}{2}\w^{\top}\left(\At_1+A_2\right)\w
\]
where
\[
\At_1 = \lambda I_d + \frac{1}{T}\sum_{t=1}^T \mathbb{I}[y_t = 1] [S_{t-1}^-]^{1/2} R_tR_t^{\top} [S_{t-1}^-]^{1/2} + \frac{1}{T}\sum_{t=1}^T \mathbb{I}[y_t = -1] [S_{t-1}^+]^{1/2}R_tR_t^{\top}[S_{t-1}^+]^{1/2}.
\]
The optimal solutions for minimizing $\L(\w)$ and $\Lh(\w)$ are given, respectively, by
\[
\w_* = (A_1 + A_2)^{-1}\a\quad\text{ and }\quad \wh_* = (\At_1 + A_2)^{-1}\a.
\]

Define $\Delta = I - A_1^{-1/2}\At_1 A_1^{-1/2}$ and write $\At_1$ in terms of $\Delta$ as
\[
\At_1  = A_1 - A_1^{1/2}\Delta A_1^{1/2}
\]
Using Lemma~\ref{thm:tmp}, it holds that $\epsilon I_d \preceq \Delta \preceq \epsilon I_d$ with probability at least $1 - \delta T$, and therefore
\[
(1 - \epsilon)(A_1+A_2) \preceq \At_1 + A_2 \preceq (1 + \epsilon)(A_1 + A_2)
\]
Denote by
\[
\Omega=(A_1+A_2)^{1/2}(A_1+A_2)^{-1}(A_1+A_2)^{1/2}-I,
\]
and according to previous analysis, we have
\begin{equation}\label{eq:temp8}
|\Omega|\leq \frac{\epsilon}{1-\epsilon} \leq 2\epsilon
\end{equation}
for $\epsilon<1/2$. Therefore,
\begin{eqnarray*}
|\wh_* - \w_*|&=&\left|\left((\At_1 + A_2)^{-1}-(A_1 + A_2)^{-1}\right)\a\right|\\
&=&|(A_1 + A_2)^{-1/2}\Omega(A_1 + A_2)^{-1/2}\a|\\
&\leq& 2\epsilon|(A_1 + A_2)^{-1}\a|\leq 2\epsilon|\w_*|
\end{eqnarray*}
and
\begin{eqnarray*}
\left|\Lh(\w_*)-\L(\w_*)\right|&=&\frac{3}{2}\left|\a^\top \left((\At_1 + A_2)^{-1}-(A_1 + A_2)^{-1}\right) \a\right|\\
&=&\frac{3}{2}\left|\a^\top \left((A_1 + A_2)^{-1/2}\Omega(A_1 + A_2)^{-1/2}\right) \a\right|\\
&\leq&2\epsilon\left|\frac{3}{2}\a^\top (A_1 + A_2)^{-1} \a\right|\leq 2\epsilon \left|\frac{1}{2}+\frac{3}{2}\a^\top (A_1 + A_2)^{-1} \a\right|=2\epsilon\L(\w_*).
\end{eqnarray*}
This theorem follows as desired.
\end{proof}

\begin{table}[!t]
\caption{Benchmark datasets}\label{tab:data}
\smallskip\centering
\begin{tabular}{|c|cc|c|cc|}
\hline
datasets &\#inst &\#feat  & datasets &\#inst &\#feat\\
\hline
\textsf{diabetes} & 768 & 8 &\textsf{w8a} & 49,749 & 300\\
\textsf{fourclass}& 862 &2 &\textsf{kddcup04} & 50,000 & 65\\
\textsf{german}& 1,000& 24 &\textsf{mnist}& 60,000  & 780\\
\textsf{splice}& 3,175 & 60 &\textsf{connect-4} & 67,557& 126\\
\textsf{usps} &9,298 & 256 &\textsf{acoustic} & 78,823 & 50\\
\textsf{letter} &15,000 & 16 &\textsf{ijcnn1} & 141,691 & 22\\
\textsf{magic04}&19,020 & 10 &\textsf{epsilon} & 400,000 & 2,000\\
\textsf{a9a} & 32,561 & 123 &\textsf{covtype} & 581,012 & 54\\
\hline
\end{tabular}
\end{table}

\noindent\textbf{Proof of Theorem~\ref{thm:main}:} For $|w_*|\leq B$, $\L(\w_*)\leq TL^*$, and $0<\epsilon\leq1/2$, we have
\begin{equation}\label{eq:temp6}
|\wh_*|\leq |\w_*|+|\wh_*-\w_*|\leq 2B
\end{equation}
from Eq.~\eqref{eq:temp3}, and we further have
\begin{equation}\label{eq:temp7}
\Lh^*(\wh_*)\leq\L(\w_*)+|\Lh^*(\wh_*)-\L(\w_*)| \leq 2\L(\w_*)\leq 2TL^*
\end{equation}
from Eq.~\eqref{eq:temp4}. Therefore, we have
\begin{equation}\label{eq:temp5}
\left|\sum_t\Lh(\w_t)-\sum_t\L(\w_*)\right|\leq\sum_t\Lh(\w_t)-\sum_t\L(\wh_*) +\left|\sum_t\Lh(\wh_*)-\sum_t\L(\w_*)\right|.
\end{equation}
We use Theorem~1 (in the main paper) to bound the first term in the above by combining Eqs.~\eqref{eq:temp6} and \eqref{eq:temp7}, and the second term can be bounded by Eq.~\eqref{eq:temp4}. This completes the proof as desired. \qed

\section{Experiments}\label{sec:exp}

We evaluate the performance of OPAUC on benchmark datasets and high-dimensional datasets in Sections~\ref{sec:exp:LSOA} and \ref{sec:exp:RPOA}, respectively. Then, we study the parameter influence in Section~\ref{sec:exp:para}.

\begin{table*}
\caption{Testing AUC (mean$\pm$std.) of OPAUC with online algorithms on benchmark datasets.  $\bullet$/$\circ$ indicates that OPAUC is significantly better/worse than the corresponding method (pairwise $t$-tests at $95\%$ significance level). }\label{tab:bench:online}
\centering
\begin{tabular}{|c|c|c|c|c|c|}
\hline
\scriptsize datasets  & \scriptsize OPAUC & \scriptsize OAM$_\text{seq}$  &\scriptsize OAM$_\text{gra}$ &\scriptsize online Uni-Exp &\scriptsize online Uni-Squ \\
\hline
\scriptsize\textsf{diabetes} &\scriptsize.8309$\pm$.0350  &\scriptsize.8264$\pm$.0367  &\scriptsize .8262$\pm$.0338 &\scriptsize.8215$\pm$.0309$\bullet$&\scriptsize .8258$\pm$.0354 \\
\scriptsize\textsf{fourclass}&\scriptsize.8310$\pm$.0251  &\scriptsize.8306$\pm$.0247  &\scriptsize.8295$\pm$.0251   &\scriptsize.8281$\pm$.0305&\scriptsize .8292$\pm$.0304 \\
\scriptsize\textsf{german}   &\scriptsize.7978$\pm$.0347 &\scriptsize.7747$\pm$.0411$\bullet$ &\scriptsize.7723$\pm$.0358$\bullet$ &\scriptsize.7908$\pm$.0367&\scriptsize.7899$\pm$.0349\\
\scriptsize\textsf{splice}&\scriptsize.9232$\pm$.0099 &\scriptsize.8594$\pm$.0194$\bullet$ &\scriptsize.8864$\pm$.0166$\bullet$ &\scriptsize.8931$\pm$.0213$\bullet$&\scriptsize.9153$\pm$.0132$\bullet$\\
\scriptsize\textsf{usps} &\scriptsize.9620$\pm$.0040 &\scriptsize.9310$\pm$.0159$\bullet$ &\scriptsize.9348$\pm$.0122$\bullet$ &\scriptsize.9538$\pm$.0045$\bullet$&\scriptsize.9563$\pm$.0041$\bullet$\\
\scriptsize\textsf{letter} &\scriptsize.8114$\pm$.0065 &\scriptsize.7549$\pm$.0344$\bullet$ &\scriptsize.7603$\pm$.0346$\bullet$ &\scriptsize.8113$\pm$.0074&\scriptsize .8053$\pm$.0081$\bullet$\\
\scriptsize\textsf{magic04} &\scriptsize.8383$\pm$.0077 &\scriptsize.8238$\pm$.0146$\bullet$ & \scriptsize.8259$\pm$.0169$\bullet$ &\scriptsize.8354$\pm$.0099$\bullet$&\scriptsize .8344$\pm$.0086$\bullet$\\
\scriptsize\textsf{a9a} &\scriptsize.9002$\pm$.0047 &\scriptsize.8420$\pm$.0174$\bullet$ & \scriptsize.8571$\pm$.0173$\bullet$ &\scriptsize.9005$\pm$.0024&\scriptsize .8949$\pm$.0025$\bullet$ \\
\scriptsize\textsf{w8a} &\scriptsize.9633$\pm$.0035 &\scriptsize.9304$\pm$.0074$\bullet$ & \scriptsize.9418$\pm$.0070$\bullet$  &\scriptsize.7693$\pm$.0986$\bullet$ &\scriptsize.8847$\pm$.0130$\bullet$ \\
\scriptsize\textsf{kddcup04} &\scriptsize.7912$\pm$.0039 &\scriptsize.6918$\pm$.0412$\bullet$ &\scriptsize.7097$\pm$.0420$\bullet$ &\scriptsize.7851$\pm$.0050$\bullet$ &\scriptsize .7850$\pm$.0042$\bullet$\\
\scriptsize\textsf{mnist} &\scriptsize.9242$\pm$.0021  &\scriptsize.8615$\pm$.0087$\bullet$ &\scriptsize.8643$\pm$.0112$\bullet$ &\scriptsize.7932$\pm$.0245$\bullet$ &\scriptsize.9156$\pm$.0027$\bullet$\\
\scriptsize\textsf{connect-4} &\scriptsize.8760$\pm$.0023 &\scriptsize.7807$\pm$.0258$\bullet$ &\scriptsize.8128$\pm$.0230$\bullet$ &\scriptsize.8702$\pm$.0025$\bullet$ &\scriptsize .8685$\pm$.0033$\bullet$\\
\scriptsize\textsf{acoustic}  &\scriptsize.8192$\pm$.0032 &\scriptsize.7113$\pm$.0590$\bullet$ &\scriptsize.7711$\pm$.0217$\bullet$ &\scriptsize.8171$\pm$.0034$\bullet$ &\scriptsize .8193$\pm$.0035\\
\scriptsize\textsf{ijcnn1} &\scriptsize.9269$\pm$.0021 &\scriptsize.9209$\pm$.0079$\bullet$ &\scriptsize.9100$\pm$.0092$\bullet$ &\scriptsize.9264$\pm$.0035 &\scriptsize .9022$\pm$.0041$\bullet$\\
\scriptsize\textsf{epsilon} &\scriptsize.9550$\pm$.0007 &\scriptsize.8816$\pm$.0042$\bullet$ &\scriptsize.8659$\pm$.0176$\bullet$ &\scriptsize.9488$\pm$.0012$\bullet$ &\scriptsize.9480$\pm$.0021$\bullet$\\
\scriptsize\textsf{covtype} &\scriptsize.8244$\pm$.0014 &\scriptsize.7361$\pm$.0317$\bullet$ &\scriptsize.7403$\pm$.0289$\bullet$ &\scriptsize .8236$\pm$.0017  &\scriptsize.8236$\pm$.0020\\
\hline
\multicolumn{2}{|c|}{\scriptsize win/tie/loss }  &\bf\scriptsize14/2/0 &\bf\scriptsize14/2/0 &\bf\scriptsize10/6/0 &\bf\scriptsize11/5/0\\
\hline
\end{tabular}
\end{table*}

\subsection{Comparison on Benchmark Data}\label{sec:exp:LSOA}
We conduct our experiments on sixteen benchmark datasets\footnote{http://www.sigkdd.org/kddcup/}$^,$\footnote{http://www.ics.uci.edu/{\textasciitilde}mlearn/MLRepository.html}$^,$\footnote{http://www.csie.ntu.edu.tw/{\textasciitilde}cjlin/libsvmtools/} as summarized in Table~\ref{tab:data}. Some datasets have been used in previous studies on AUC optimization, whereas the other are large ones requiring one-pass procedure. The features have been scaled to $[-1,1]$ for all datasets. Multi-class datasets have been transformed into binary ones by randomly partitioning classes into two groups, where each group contains the same number of classes.

In addition to state-of-the-art online AUC approaches {\bf OAM$_\text{seq}$} and {\bf OAM$_\text{gra}$} \citep{Zhao:Hoi:Jin:Yang2011}, we also compare with:
\begin{itemize}
\item {\bf online Uni-Exp}: An online learning algorithm which optimizes the (weighted) univariate exponential loss~\citep{Kotlowski:Dembczynski:Hullermeier2011};
\item {\bf online Uni-Squ}: An online learning algorithm which optimizes the (weighted) univariate square  loss;
\item {\bf SVM-perf}: A batch learning algorithm which directly optimizes AUC \citep{Joachims2005};

\begin{table*}
\caption{Testing AUC (mean$\pm$std.) of OPAUC with batch algorithms on benchmark datasets. $\bullet$/$\circ$ indicates that OPAUC is significantly better/worse than the corresponding method (pairwise $t$-tests at $95\%$ significance level). }\label{tab:bench:batch}
\centering
\begin{tabular}{|c|c|c|c|c|c|c|c|c|}
\hline
\scriptsize datasets  &\scriptsize OPAUC &\scriptsize SVM-perf  &\scriptsize batch SVM-OR &\scriptsize batch LS-SVM &\scriptsize batch Uni-Log  &\scriptsize batch Uni-Squ\\
\hline
\scriptsize\textsf{diabetes} &\scriptsize.8309$\pm$.0350
&\scriptsize.8325$\pm$.0220 &\scriptsize.8326$\pm$.0328  &\scriptsize.8325$\pm$.0329 &\scriptsize.8330$\pm$.0322
&\scriptsize.8332$\pm$.0323\\

\scriptsize\textsf{fourclass}&\scriptsize.8310$\pm$.0251
&\scriptsize.8221$\pm$.0381  &\scriptsize.8305$\pm$.0311  &\scriptsize.8309$\pm$.0309 &\scriptsize.8288$\pm$.0307
&\scriptsize.8297$\pm$.0310\\

\scriptsize\textsf{german}   &\scriptsize.7978$\pm$.0347
&\scriptsize.7952$\pm$.0340 &\scriptsize.7935$\pm$.0348  &\scriptsize.7994$\pm$.0343 &\scriptsize.7995$\pm$.0344
&\scriptsize.7990$\pm$.0342\\

\scriptsize\textsf{splice}   &\scriptsize.9232$\pm$.0099
&\scriptsize.9235$\pm$.0091   &\scriptsize.9239$\pm$.0089  &\scriptsize.9245$\pm$.0092$\circ$ &\scriptsize.9208$\pm$.0107$\bullet$
&\scriptsize.9211$\pm$.0107$\bullet$\\

\scriptsize\textsf{usps} &\scriptsize.9620$\pm$.0040
&\scriptsize.9600$\pm$.0054$\bullet$  &\scriptsize.9630$\pm$.0047$\circ$ &\scriptsize.9634$\pm$.0045$\circ$ &\scriptsize.9637$\pm$.0041$\circ$
&\scriptsize.9617$\pm$.0043\\

\scriptsize\textsf{letter} &\scriptsize.8114$\pm$.0065
&\scriptsize.8028$\pm$.0074$\bullet$  &\scriptsize.8144$\pm$.0064$\circ$ &\scriptsize.8124$\pm$.0065$\circ$ &\scriptsize.8121$\pm$.0061
&\scriptsize.8112$\pm$.0061\\

\scriptsize\textsf{magic04} &\scriptsize.8383$\pm$.0077
&\scriptsize.8427$\pm$.0078$\circ$  &\scriptsize.8426$\pm$.0074$\circ$ &\scriptsize.8379$\pm$0.0078 &\scriptsize.8378$\pm$.0073
&\scriptsize.8338$\pm$.0073$\bullet$\\

\scriptsize\textsf{a9a} &\scriptsize.9002$\pm$.0047
&\scriptsize.9033$\pm$.0039 &\scriptsize.9009$\pm$.0036 &\scriptsize.8982$\pm$.0028$\bullet$ &\scriptsize.9033$\pm$.0025$\circ$
&\scriptsize.8967$\pm$.0028$\bullet$\\

\scriptsize\textsf{w8a} &\scriptsize.9633$\pm$.0035
&\scriptsize.9626$\pm$.0042  &\scriptsize.9495$\pm$.0082$\bullet$ &\scriptsize.9495$\pm$.0092$\bullet$ &\scriptsize.9421$\pm$.0062$\bullet$
&\scriptsize.9075$\pm$.0104$\bullet$\\

\scriptsize\textsf{kddcup04} &\scriptsize.7912$\pm$.0039
&\scriptsize.7935$\pm$.0037$\circ$  &\scriptsize.7903$\pm$.0039$\bullet$ &\scriptsize.7898$\pm$.0039$\bullet$ &\scriptsize.7900$\pm$.0039$\bullet$
&\scriptsize.7926$\pm$.0038\\

\scriptsize\textsf{mnist} &\scriptsize.9242$\pm$.0021
&\scriptsize.9338$\pm$.0022$\circ$  &\scriptsize.9340$\pm$.0020$\circ$ &\scriptsize.9336$\pm$.0025$\circ$ &\scriptsize.9334$\pm$.0021$\circ$
&\scriptsize.9279$\pm$.0021$\circ$\\

\scriptsize\textsf{connect-4} &\scriptsize.8760$\pm$.0023
&\scriptsize.8794$\pm$.0024$\circ$  &\scriptsize.8749$\pm$.0025$\bullet$ &\scriptsize.8739$\pm$.0026$\bullet$ &\scriptsize.8784$\pm$.0026$\circ$
&\scriptsize.8760$\pm$.0024\\

\scriptsize\textsf{acoustic}  &\scriptsize.8192$\pm$.0032
&\scriptsize.8102$\pm$.0032$\bullet$  &\scriptsize.8262$\pm$.0032$\circ$ &\scriptsize.8210$\pm$.0033$\circ$ &\scriptsize.8253$\pm$.0032$\circ$
&\scriptsize.8222$\pm$.0031$\circ$\\

\scriptsize\textsf{ijcnn1} &\scriptsize.9269$\pm$.0021
&\scriptsize.9314$\pm$.0025$\circ$  &\scriptsize.9337$\pm$.0024$\circ$ &\scriptsize.9320$\pm$.0037$\circ$  &\scriptsize.9282$\pm$.0023$\circ$
&\scriptsize.9038$\pm$.0025$\bullet$\\

\scriptsize\textsf{epsilon} &\scriptsize.9550$\pm$.0007
&\scriptsize.8640$\pm$.0049$\bullet$ &\scriptsize.8643$\pm$.0053$\bullet$ &\scriptsize.8644$\pm$.0050$\bullet$ &\scriptsize.8647$\pm$.0150$\bullet$
&\scriptsize.8653$\pm$.0073$\bullet$\\

\scriptsize\textsf{covtype} &\scriptsize.8244$\pm$.0014
&\scriptsize.8271$\pm$.0011$\circ$  &\scriptsize.8248$\pm$.0013 &\scriptsize.8222$\pm$.0014$\bullet$ &\scriptsize.8246$\pm$.0010
&\scriptsize.8242$\pm$.0012\\
\hline
\multicolumn{2}{|c|}{\scriptsize win/tie/loss } &\bf\scriptsize4/6/6 &\bf\scriptsize4/6/6 &\scriptsize\bf6/4/6 &\bf\scriptsize4/6/6 &\scriptsize\bf6/8/2\\
\hline
\end{tabular}
\end{table*}

\item {\bf batch SVM-OR}: A batch learning algorithm which optimizes the pairwise hinge loss \citep{Joachims2006};
\item {\bf batch LS-SVM}: A batch learning algorithm which optimizes the pairwise square loss;
\item {\bf batch Uni-Log}: A batch learning algorithm which optimizes the (weighted) univariate logistic loss~\citep{Kotlowski:Dembczynski:Hullermeier2011};
\item {\bf batch Uni-Squ}: A batch learning algorithm which optimizes the (weighted) univariate square loss.
\end{itemize}

All experiments are performed with Matlab 7 on a node of computational cluster with 16 CPUs (Intel Xeon Due Core 3.0GHz) running RedHat Linux Enterprise 5 with 48GB main memory. For batch algorithms, due to memory limit, 8,000 training examples are randomly chosen  if training data size exceeds 8,000, whereas only 2,000 training examples are used for the \textsf{epsilon} dataset because of its high dimension.

Five-fold cross-validation is executed on training sets to decide the learning rate $\eta_t\in2^{[-12:10]}$ for online algorithms, the regularized parameter $\lambda\in2^{[-10:2]}$ for OPAUC and $\lambda\in2^{[-10:10]}$ for batch algorithms. For OAM$_\text{seq}$ and OAM$_\text{gra}$, the buffer sizes are fixed to be 100 as recommended in \citep{Zhao:Hoi:Jin:Yang2011}. For univariate approaches, the weights (i.e., class ratios) are chosen as done in \citep{Kotlowski:Dembczynski:Hullermeier2011}.

The performances of the compared methods are evaluated by five trials of 5-fold cross validation, where the AUC values are obtained by averaging over these 25 runs. Table~\ref{tab:bench:online} shows that OPAUC is significant better than the other four online algorithms OAM$_\text{seq}$, OAM$_\text{gra}$, online Uni-Exp and online Uni-Squ, particularly for large datasets. The win/tie/loss counts show that OPAUC is clearly superior to these online algorithms, as it wins for most times and never loses. Table~\ref{tab:bench:batch} shows shows that OPAUC is highly competitive to the other five batch learning algorithms; this is impressive because these batch algorithms require storing the whole training dataset whereas OPAUC does not store training data. Additionally, batch LS-SVM which optimizes the square loss is comparable to the other batch algorithms, verifying our argument that square loss is effective for AUC optimization.

\begin{figure}[!t]
\centering
\begin{minipage}{2.5in}
\includegraphics[width=3.5in]{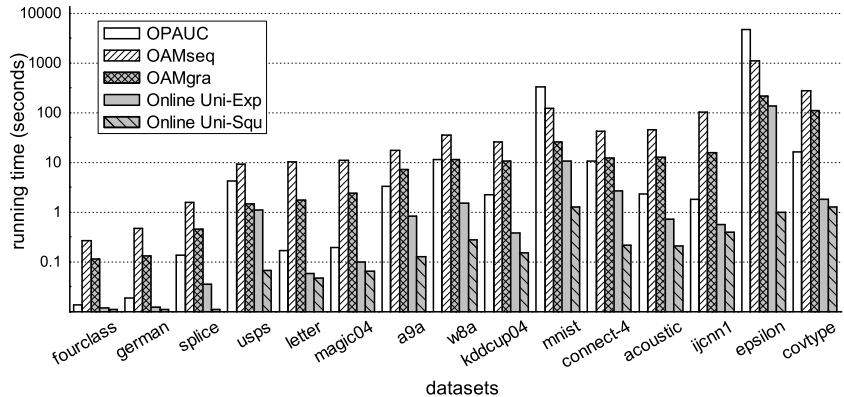}
\end{minipage}\vspace{-0.7in}
\caption{Comparison of the running time (in seconds) of OPAUC and online learning algorithms on benchmark data sets. Notice that the $y$-axis is in log-scale.}\label{fig:running-time}
\end{figure}
\begin{table}[!t]
\caption{High-dimensional datasets}\label{tab:highdata}
\smallskip\centering
\begin{tabular}{|c|cc|c|cc|}
\hline
 datasets &\#inst &\#feat  & datasets &\#inst &\#feat\\
\hline
\textsf{sector} &9,619 &55,197 & \textsf{news20.binary} &19,996  &1,355,191\\
\textsf{sector.lvr} &9,619 &55,197 &\textsf{rcv1v2} &23,149  &47,236 \\
\textsf{news20} &15,935 &62,061 &\textsf{ecml2012} &456,886  &98,519\\
\hline
\end{tabular}
\end{table}

We also compare the running time of OPAUC and the online algorithms OAM$_\text{seq}$, OAM$_\text{gra}$, online Uni-Exp and online Uni-Squ, and the average CPU time (in seconds) are shown in Figure~\ref{fig:running-time}. As expected, online Uni-Squ and online Uni-Exp takes the least time cost because they optimize on single-instance (univariate) loss, whereas the other algorithms work by optimizing pairwise loss. On most datasets, the running time of OPAUC is competitive to OAM$_\text{seq}$ and OAM$_\text{gra}$, except on the \textsf{mnist} and \textsf{epsilon} datasets which have the highest dimension in Table~\ref{tab:data}.

\subsection{Comparison on High-Dimensional Data}\label{sec:exp:RPOA}

Next, we study the performance of using low-rank matrices to approximate the full covariance matrices, denoted by OPAUCr.  Six datasets\footnote{http://www.csie.ntu.edu.tw/{\textasciitilde}cjlin/libsvmtools/}$^,$\footnote{\small{http://www.ecmlpkdd2012.net/discovery-challenge}} with nearly or more than 50,000 features are used, as summarized in Table~\ref{tab:highdata}. The \textsf{news20.binary} dataset contains two classes, different from \textsf{news20} dataset. The original \textsf{news20} and \textsf{sector} are multi-class datesets; in our experiments, we randomly group the multiple classes into two meta-classes each containing the same number of classes, and we also use the \textsf{sector.lvr} dataset which regards the largest class as positive whereas the union of other classes as negative. The original \textsf{ecml2012} and \textsf{rcv1v2} are multi-label datasets; in our experiments, we only consider the label with the largest population, and remove the features in \textsf{ecml2012} dataset that take zero values for all instances.

\begin{table*}
\caption{Testing AUC (mean$\pm$std.) of OPAUCr with online methods on high-dimensional datasets. $\bullet$/$\circ$ indicates that OPAUCr is significantly better/worse than the corresponding method (pairwise $t$-tests at $95\%$ significance level). `N/A' means that no result was obtained after running out $10^6$ seconds (about 11.6 days).}\label{tab:result2}
\centering
\begin{tabular}{|c|c|c|c|c|c|c|c|}
\hline
\scriptsize datasets & \scriptsize\textsf{sector} & \scriptsize\textsf{sector.lvr} &\scriptsize\textsf{news20} &\scriptsize \textsf{news20.binary}  &\scriptsize\textsf{rcv1v2 } &\scriptsize \textsf{ecml2012}\\
\hline
\scriptsize OPAUCr            &\scriptsize.9292$\pm$.0081
&\scriptsize.9962$\pm$.0011   &\scriptsize  .8871$\pm$.0083
&\scriptsize.6389$\pm$.0136   &\scriptsize.9686$\pm$.0029
&\scriptsize .9828$\pm$.0008\\
\scriptsize OAM$_\text{seq}$  &\scriptsize.9163$\pm$.0087$\bullet$ &\scriptsize.9965$\pm$.0064   &\scriptsize.8543$\pm$.0099$\bullet$ &\scriptsize.6314$\pm$.0131$\bullet$ &\scriptsize.9686$\pm$.0026
&\scriptsize N/A\\
\scriptsize OAM$_\text{gra}$  &\scriptsize.9043$\pm$.0100$\bullet$ &\scriptsize.9955$\pm$.0059   &\scriptsize.8346$\pm$.0094$\bullet$ &\scriptsize.6351$\pm$.0135$\bullet$ &\scriptsize.9604$\pm$.0025$\bullet$
&\scriptsize.9657$\pm$.0055$\bullet$\\
\scriptsize  online Uni-Exp   &\scriptsize.9215$\pm$.0034$\bullet$ &\scriptsize.9969$\pm$.0093   &\scriptsize.8880$\pm$.0047 &\scriptsize.6347$\pm$.0092$\bullet$ &\scriptsize.9822$\pm$.0042$\circ$
&\scriptsize .9820$\pm$.0016$\bullet$ \\
\scriptsize  online Uni-Squ   &\scriptsize.9203$\pm$.0043$\bullet$ &\scriptsize$.9669\pm$.0260 &\scriptsize$.8878\pm$.0066 &\scriptsize.6237$\pm$.0104$\bullet$ &\scriptsize.9818$\pm$.0014 &\scriptsize.9530$\pm$.0041$\bullet$ \\
\scriptsize OPAUC$^\text{f}$  &\scriptsize.6228$\pm$.0145$\bullet$ &\scriptsize.6813$\pm$.0444$\bullet$ &\scriptsize.5958$\pm$.0118$\bullet$ &\scriptsize.5068$\pm$.0086$\bullet$ &\scriptsize.6875$\pm$.0101$\bullet$
&\scriptsize.6601$\pm$.0036$\bullet$\\
\scriptsize  OPAUC$^\text{rp}$&\scriptsize.7286$\pm$.0619$\bullet$ &\scriptsize.9863$\pm$.0258$\bullet$ &\scriptsize.7885$\pm$.0079$\bullet$ &\scriptsize.6212$\pm$.0072$\bullet$ &\scriptsize.9353$\pm$.0053$\bullet$
&\scriptsize.9355$\pm$.0047$\bullet$\\
\scriptsize OPAUC$^\text{pca}$&\scriptsize.8853$\pm$.0114$\bullet$ &\scriptsize.9893$\pm$.0288$\bullet$ &\scriptsize.8878$\pm$.0115
&\scriptsize N/A &\scriptsize.9752$\pm$.0020$\circ$
&\scriptsize N/A\\
\hline
\end{tabular}
\end{table*}

\begin{figure}[!t]
\centering
\begin{minipage}{2.5in}
\includegraphics[width=3.5in]{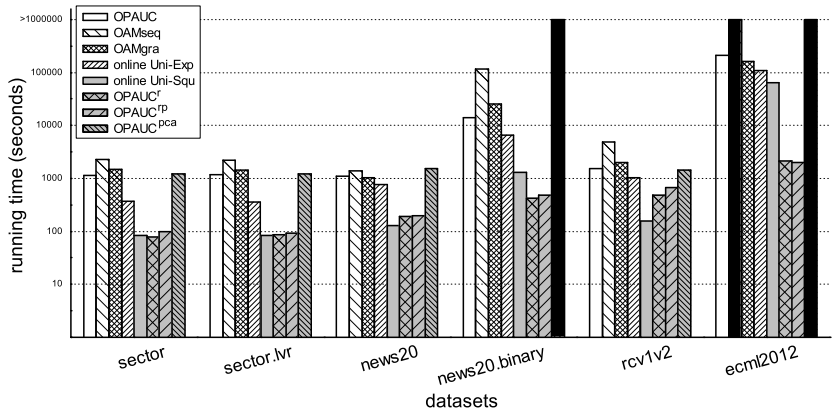}
\end{minipage}
\caption{Comparison of the running time on high dimensional datasets. Full black columns imply that no results were returned after running out the maximal running time.}\label{fig:runningtime2}
\end{figure}

Besides the online algorithms OAM$_\text{seq}$, OAM$_\text{gra}$, online Uni-Exp and online Uni-Squ, we also evaluate three variants of OPAUC to study the effectiveness of approximating full covariance matrices with low-rank matrices:
\begin{itemize}
\item {\bf OPAUC$^\text{f}$}: Randomly selects $1,000$-dim features and then works with full covariance matrices;
\item {\bf OPAUC$^\text{rp}$}: Projects into a $1,000$-dim feature space by Random Projection, and then works with full covariance matrices;
\item {\bf OPAUC$^\text{pca}$}: Projects into a $1,000$-dim feature space obtained by Principle Component Analysis, and then works with full covariance matrices.
\end{itemize}

Similar to Section 5.1, five-fold cross validation is executed on training sets to decide the learning rate $\eta_t\in2^{[-12:10]}$ and the regularization parameter $\lambda\in2^{[-10:2]}$. Due to memory and computational limit, the buffer sizes are set to 50 for OAM$_\text{seq}$ and OAM$_\text{gra}$, and the rank $\tau$ of OPAUCr is also set to 50. The performances of the compared methods are evaluated by five trials of 5-fold cross validation, where the AUC values are obtained by averaging over these 25 runs.

The comparison results are summarized in Table~\ref{tab:result2} and the average running time is shown in Figure~\ref{fig:runningtime2}. These results clearly show that our approximate OPAUCr approach is superior to the other compared methods. Compared with OAM$_\text{seq}$ and OAM$_\text{gra}$, the running time costs are comparable whereas the performance of OPAUCr is better. Online Uni-Squ and Uni-Exp are more efficient than OPAUCr because it optimizes univariate loss, but the performance of OPAUCr is highly competitive or better, except on \textsf{rcv1v2}, the only dataset with less than 50,000 features. Compared with the three variants, OPAUC$^\text{f}$ and OPAUC$^\text{rp}$ are more efficient, but with much worse performances. OPAUC$^\text{pca}$ achieves a better performance on \textsf{rcv1v2}, but it is worse on datasets with more features; particularly, on the two datasets with the largest number of features, OPAUC$^\text{pca}$ cannot return results even after running out $10^6$ seconds (almost 11.6 days). Our approximate OPAUCr approach is significantly better than all the other methods (if they return results) on the two datasets with the largest number of features: \textsf{news.binary} with more than 1 million features, and \textsf{ecml2012} with nearby 100 thousands features. These observations validate the effectiveness of the low-rank approximation used by OPAUCr for handling high-dimensional data.

\subsection{Parameter Influence}\label{sec:exp:para}
We study the influence of parameters in this section. Figure~\ref{fig1} shows that stepsize $\eta_t$ should not be set to values bigger than $1$, whereas there is a relatively big range between $[2^{-12}, 2^{-4}]$ where OPAUC achieves good results. Figures~\ref{fig2} shows that OPAUC is not sensitive to the value of regularization parameter $\lambda$ given that it is not set with a big value. Figure~\ref{fig3} shows that OPAUCr is not sensitive to the values of rank $\tau$, and it works well even when $\tau=50$; this verifies Theorem~\ref{thm:main} that a relatively small $\tau$ value suffices to lead to a good approximation performance. Figure~\ref{fig6} compares studies the influence of the iterations for OPAUC, OAM$_\text{seq}$ and OAM$_\text{gra}$, and it is observable that OPAUC convergence faster than the other two algorithms, which verifies our theoretical argument in Section~\ref{sec:theory}.

\begin{figure}[!t]
\centering
\begin{minipage}{2in}
\includegraphics[width=2.25in]{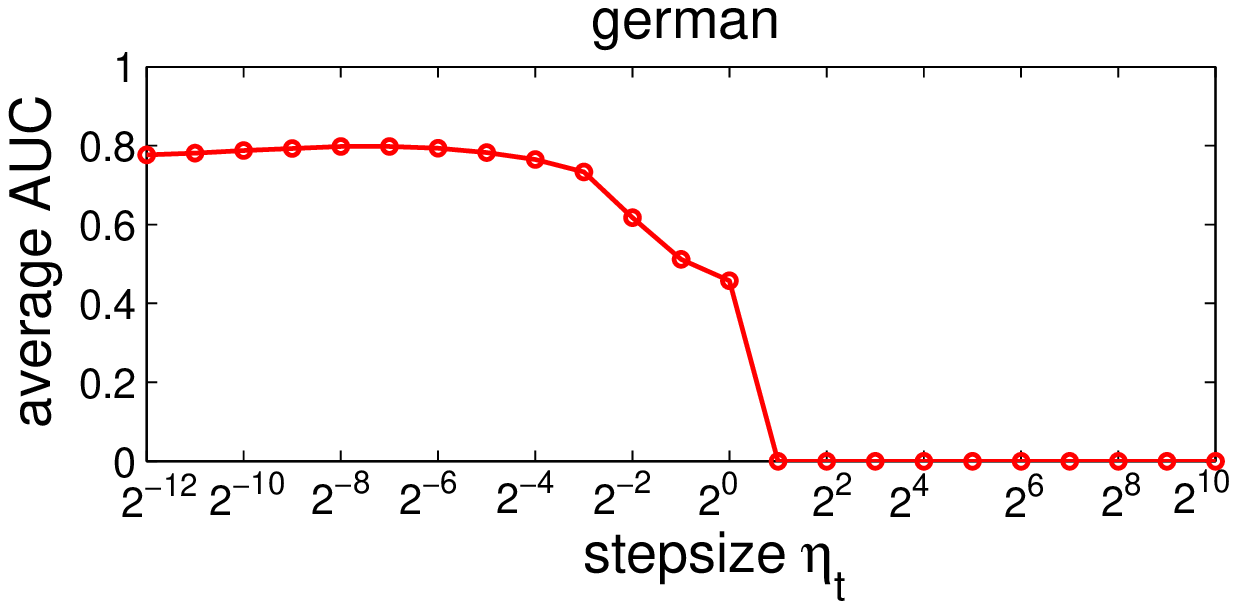}\\
\end{minipage}
\begin{minipage}{2in}
\includegraphics[width=2.25in]{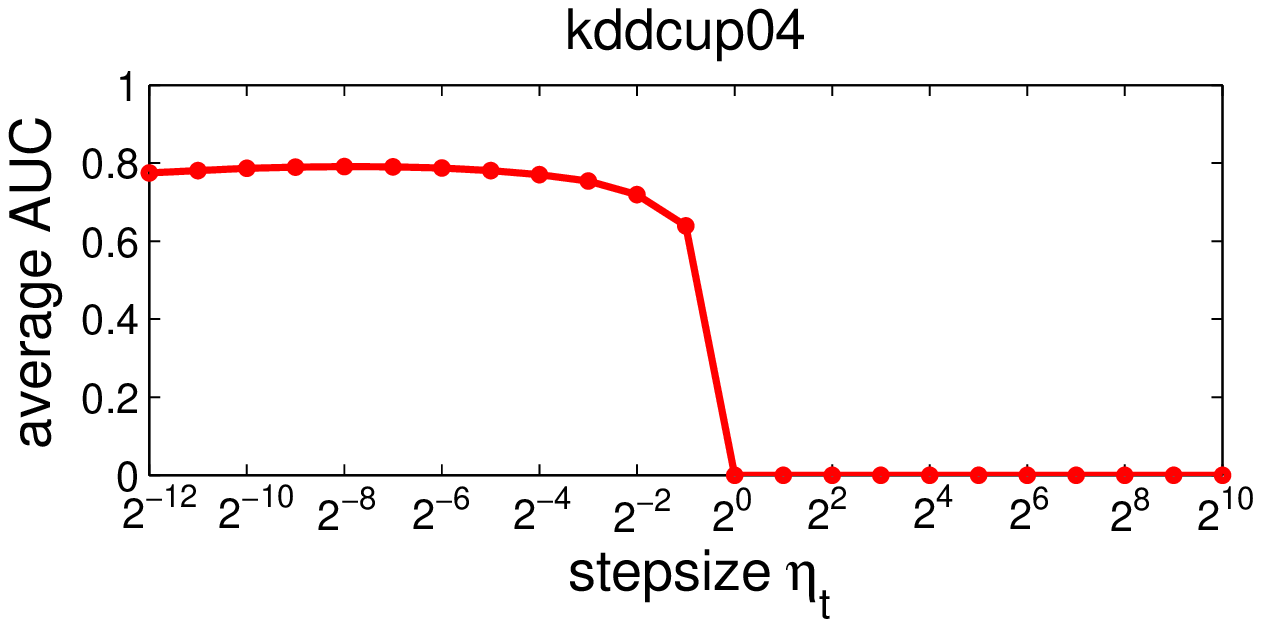}\\
\end{minipage}
\begin{minipage}{2in}
\includegraphics[width=2.25in]{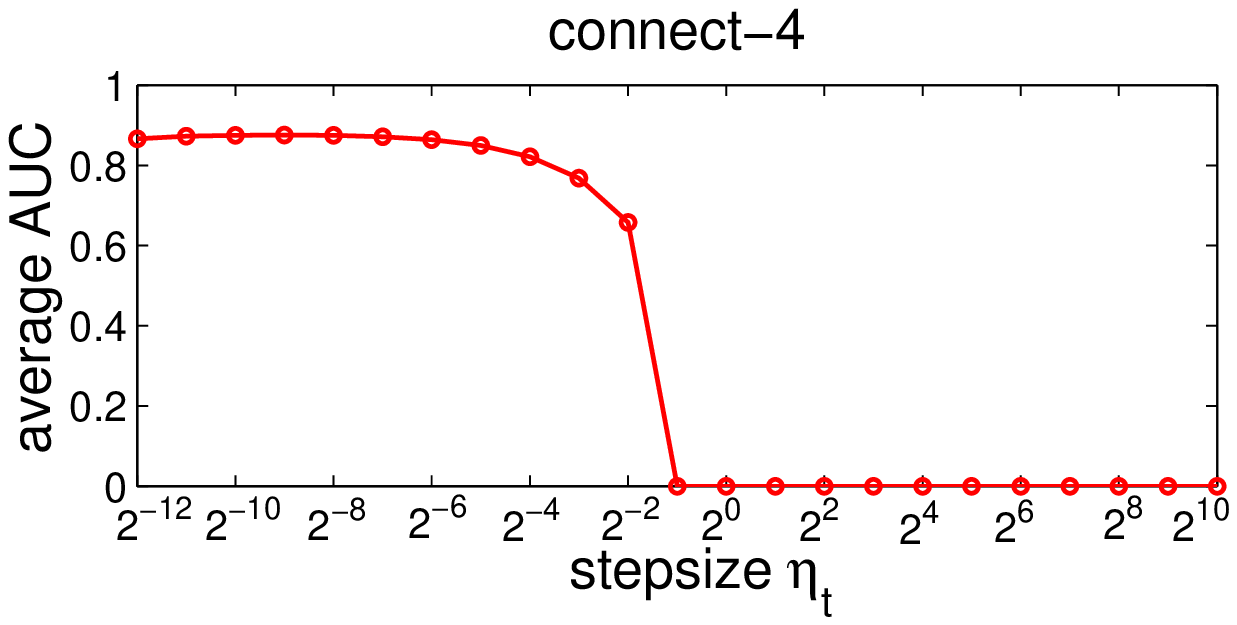}\\
\end{minipage}
\begin{minipage}{2in}
\includegraphics[width=2.25in]{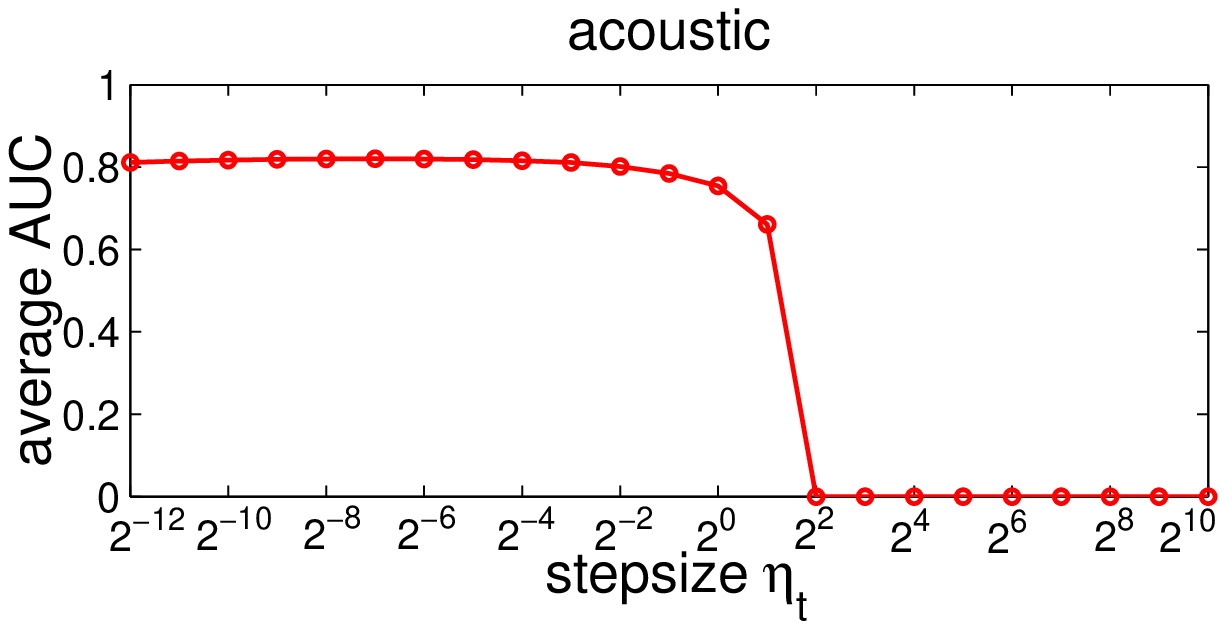}\\
\end{minipage}\vspace{-0.1in}
\caption{Influence of stepsize $\eta_t$}\label{fig1}
\end{figure}

\begin{figure}[!t]
\centering
\begin{minipage}{2in}
\includegraphics[width=2.25in]{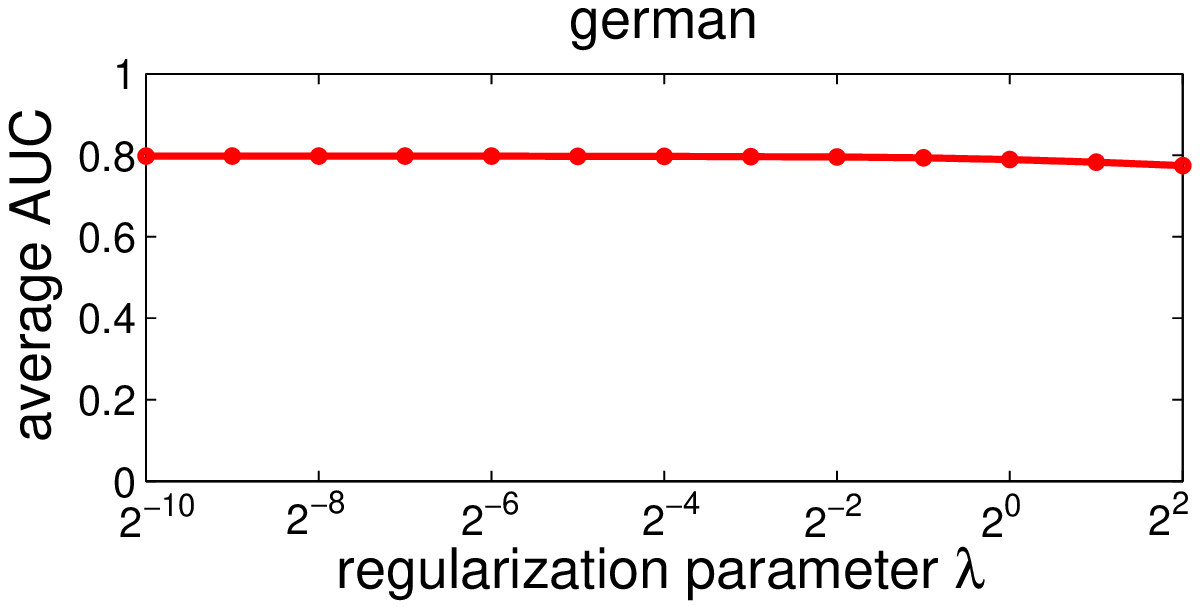}\\
\end{minipage}
\begin{minipage}{2in}
\includegraphics[width=2.25in]{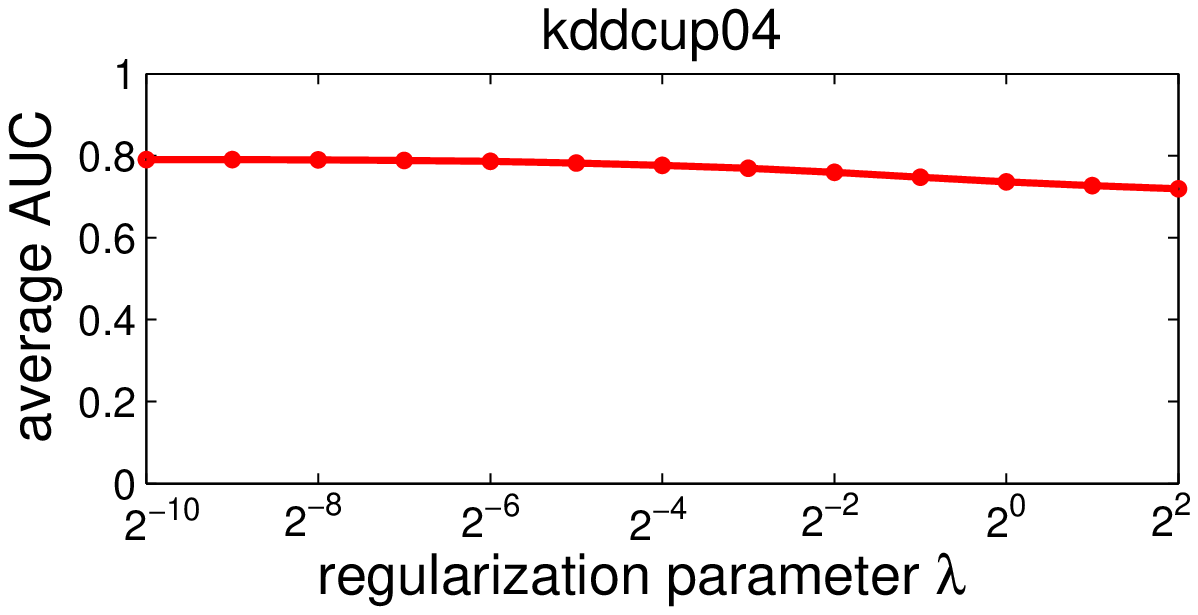}\\
\end{minipage}
\begin{minipage}{2in}
\includegraphics[width=2.25in]{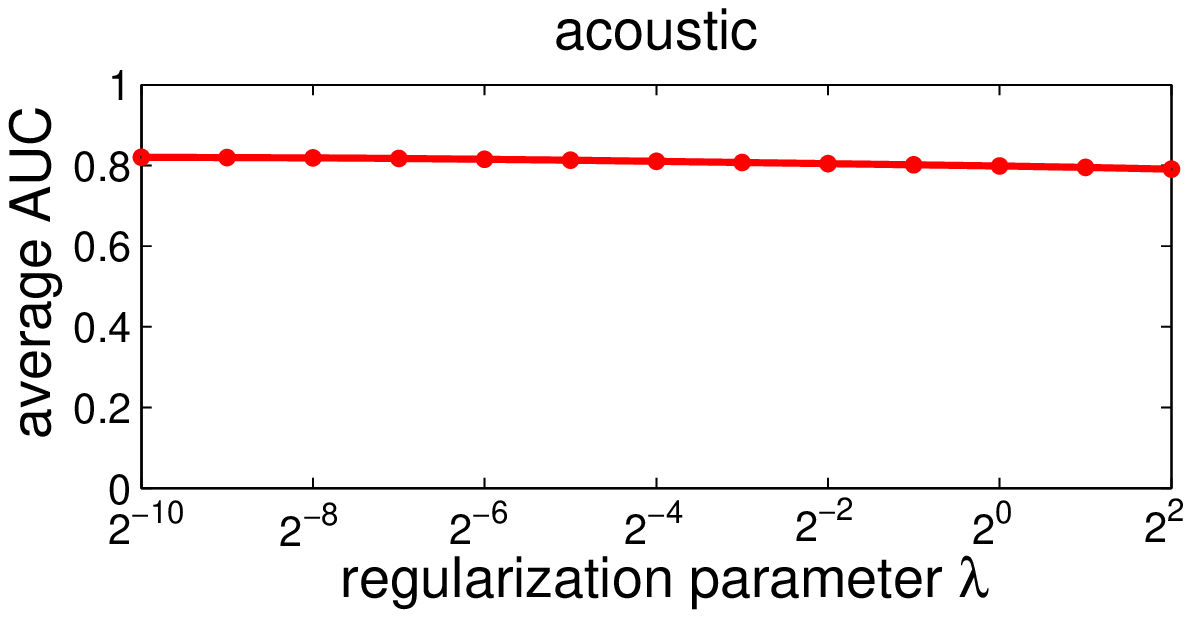}\\
\end{minipage}
\begin{minipage}{2in}
\includegraphics[width=2.25in]{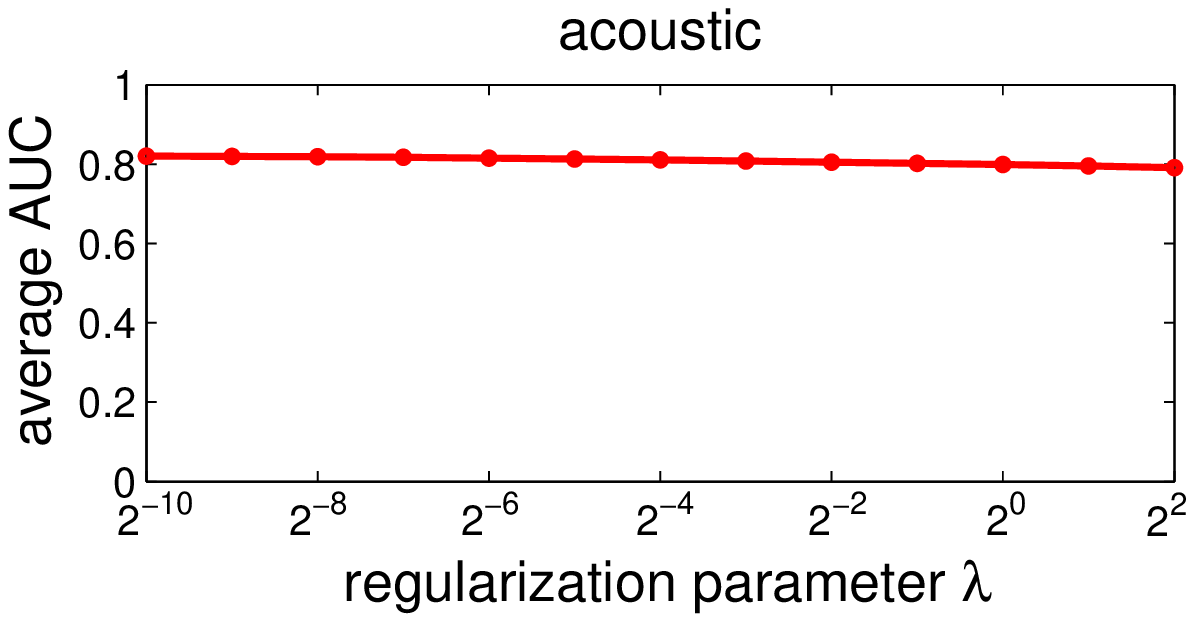}\\
\end{minipage}\vspace{-0.1in}
\caption{Influence of regularization parameter $\lambda$}\label{fig2}
\end{figure}

\begin{figure}[!t]
\centering
\begin{minipage}{2in}
\includegraphics[width=2.5in]{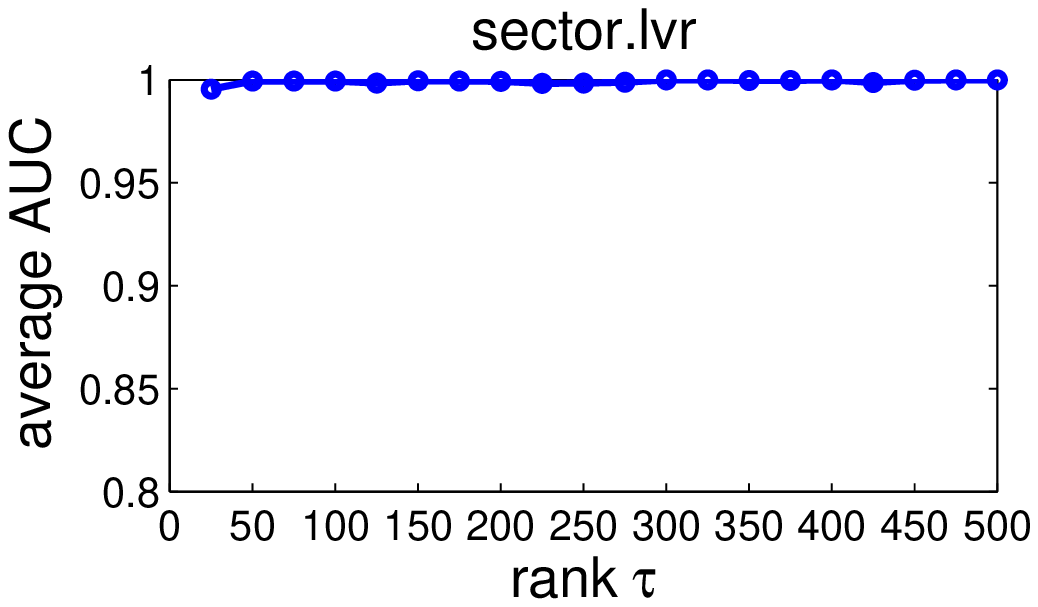}\\
\end{minipage}
\begin{minipage}{2in}
\includegraphics[width=2.5in]{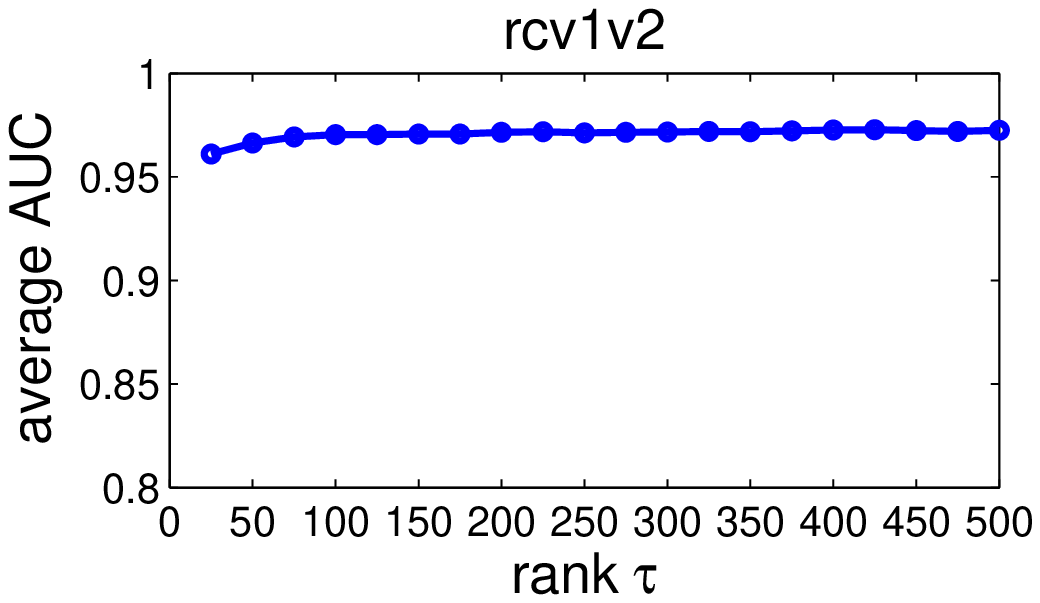}\\
\end{minipage}\vspace{-0.1in}
\caption{Influence of rank $\tau$}\label{fig3}
\end{figure}

\begin{figure}[!t]
\centering
\begin{minipage}{2in}
\includegraphics[width=2.5in]{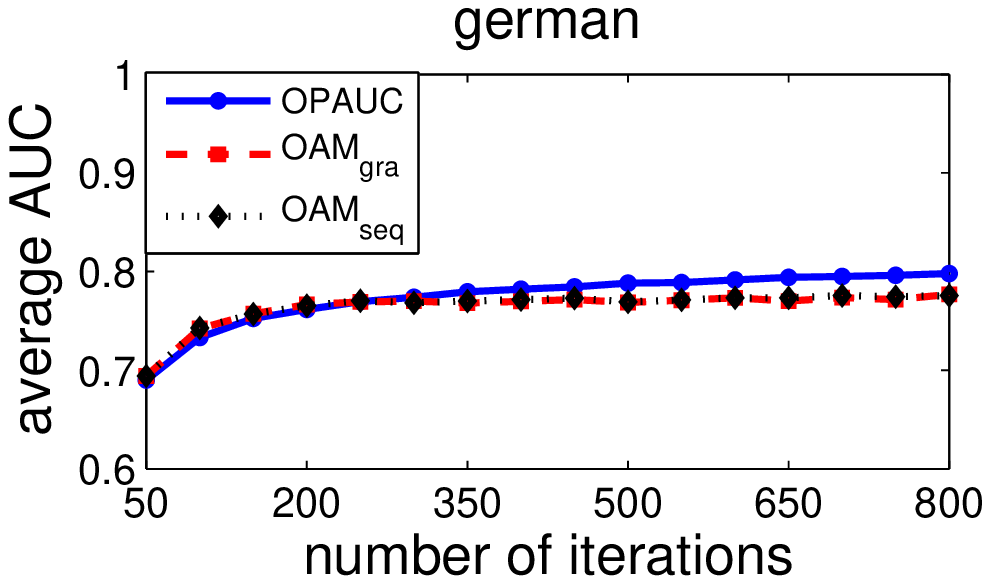}\\
\end{minipage}
\begin{minipage}{2in}
\includegraphics[width=2.5in]{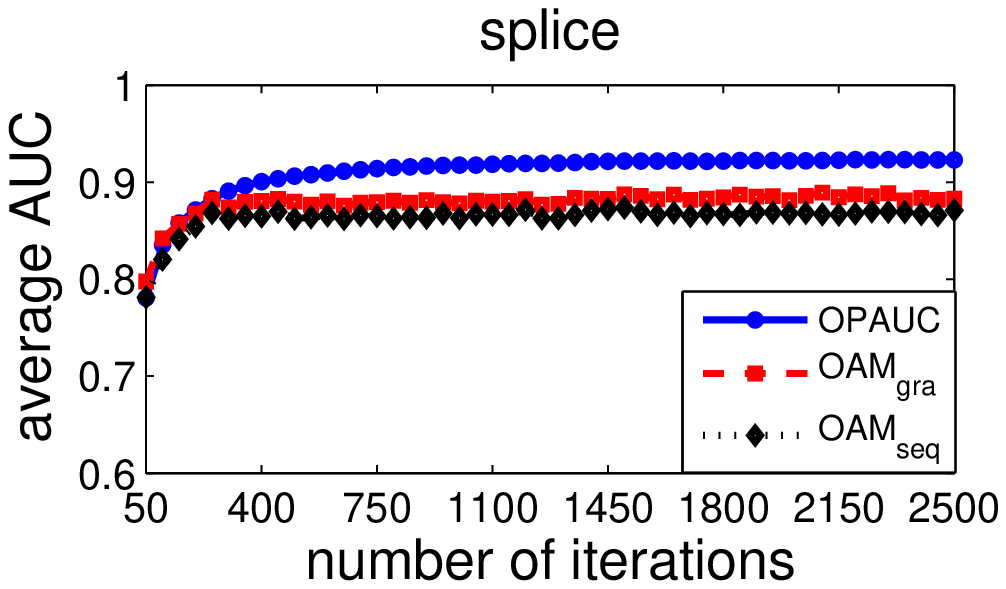}\\
\end{minipage}\vspace{-0.1in}
\caption{Convergence comparisons of OPAUC, OAM$_\text{seq}$ and OAM$_\text{gra}$}\label{fig6}
\end{figure}

\section{Conclusion}\label{sec:con}
In this paper, we study one-pass AUC optimization that requires going through the training data only once, without storing the entire dataset. Here, a big challenge lies in the fact that AUC is measured by a sum of losses defined over pairs of instances from different classes. We propose the OPAUC approach, which employs a square loss and requires the storing of only the first and second-statistics for the received training examples. A nice property of OPAUC is that its storage requirement is O($d^2$), where $d$ is the dimension of data, independent from the number of training examples. To handle high-dimensional data, we develop an approximate strategy by using low-rank matrices. The effectiveness of our proposed approach is verified both theoretically and empirically. In particular, the performance of OPAUC is significantly better than state-of-the-art online AUC optimization approaches, even highly competitive to batch learning approaches; the approximate OPAUC is significantly better than all compared methods on large datasets with one hundred thousands or even more than one million features. An interesting future issue is to develop one-pass AUC optimization approaches not only with a performance comparable to batch approaches, but also with an efficiency comparable to univariate loss optimization approaches.

\bibliography{reference}
\end{document}